\documentclass[lettersize,11pt]{extarticle}
\usepackage{amsfonts, amsthm, amssymb, bbm, amsmath}
\usepackage{algorithmic}
\usepackage{algorithm}
\usepackage{array}
\usepackage[margin=1.2in]{geometry}
\usepackage{setspace}
\usepackage{textcomp}
\usepackage{stfloats}
\usepackage{url}
\usepackage{verbatim}
\usepackage{graphicx}
\usepackage{multirow}
\usepackage{booktabs}
\usepackage{cite}
\usepackage{subcaption}
\usepackage{xcolor}
\usepackage{hyperref}
\usepackage{cleveref}
\allowdisplaybreaks

\usepackage{bbm}
\usepackage[labelformat=simple]{subcaption}
\def\W{\mathbf{W}}

\def\x{\mathbf{x}}
\def\y{\mathbf{y}}
\def\h{\mathbf{h}}

\def\R{\mathbb{R}}

\def\C{\mathbf{C}}
\def\bzeta{\boldsymbol{\zeta}}
\def\lC{\overline{\mathbf{C}}}
\def\ld{\overline{\mathbf{d}}}

\def\X{\mathcal{X}}
\def\Y{\mathcal{Y}}
\def\bmu{\boldsymbol{\mu}}
\def\bSigma{\boldsymbol{\Sigma}}
\def\N{\mathcal{N}}
\def\diag{\mathrm{diag}}
\def\z{\mathbf{z}}

\def\1{\mathbbm{1}}

\newtheorem{definition}{Definition}
\newtheorem{lemma}{Lemma}

\newtheorem{proposition}{Proposition}

\begin{document}

\title{ReLU Networks as Random Functions: Their Distribution in Probability Space}

\author{Shreyas Chaudhari \and Jos\'e M. F. Moura
}
\date{
Carnegie Mellon University\\
Electrical and Computer Engineering}

\maketitle


\begin{abstract}
This paper presents a novel framework for understanding trained ReLU networks as random, affine functions, where the randomness is induced by the distribution over the inputs. By characterizing the probability distribution of the network’s activation patterns, we derive the discrete probability distribution over the affine functions realizable by the network. We extend this analysis to describe the probability distribution of the network’s outputs. Our approach provides explicit, numerically tractable expressions for these distributions in terms of Gaussian orthant probabilities. Additionally, we develop approximation techniques to identify the support of affine functions a trained ReLU network can realize for a given distribution of inputs. Our work provides a framework for understanding the behavior and performance of ReLU networks corresponding to stochastic inputs, paving the way for more interpretable and reliable models.
\end{abstract}

\section{Introduction}
Neural networks have become foundational to modern machine learning. Their ability to learn from copious amounts of data and effectively model complex, non-linear relationships has enabled them to outperform classical methods in many application domains spanning computer vision, natural language processing, and medicine \cite{lecun2015deep}. The rapid advancements and widespread adoption of deep learning have motivated extensive research into the fundamental aspects of neural networks, including their expressive power \cite{cohen2016expressive}, training dynamics \cite{jacot2018neural}, and feature representations \cite{kansizoglou2021deep}. Several works have also studied uncertainty quantification in neural networks using moment propagation techniques \cite{bibi2018analytic, wright2024analytic}. Nevertheless, the inner workings of deep neural networks still remain largely elusive, leading to neural networks being treated as black boxes in practice. The resulting lack of transparency poses significant challenges, particularly in safety-critical and high-stakes applications, where understanding the behavior, robustness, and decision-making process of the deployed neural network is paramount. 

We show that for the decision-making process and robustness of a prominent class of neural networks, namely those with continuous, piecewise linear activations (e.g., ReLU, Leaky ReLU), can be formulated \textit{exactly} using numerically tractable expressions. Such networks are piecewise affine functions, meaning they partition the input space into distinct regions, and behave as a consistent affine function within each region. Since the number of these regions is finite, the network can be characterized by a finite set of affine functions that it can assume over the input space. The key question we explore in this work is: 

\textit{Given that inputs to the network are randomly distributed according to some distribution, what is the probability of a ReLU network realizing a particular affine function from the set of possible functions that it can assume?}

The randomness of the inputs effectively yields a probability distribution over the affine functions the network can assume. Addressing the question proposed above is crucial for developing a deeper understanding of neural networks, especially as it pertains to their robustness and fairness. For example, such analysis is critical in high-stakes scenarios, where it is necessary to guarantee that the probability of a potentially ``dangerous'' affine function being realized is negligible. Such analysis can also help assess the likelihood that the function applied by a neural network changes when sensitive input features are randomly adjusted, ensuring that the models remain fair and stable in their behavior. Our key contributions are summarized as follows: 
\begin{itemize}
    \item We exactly derive the probability that a neural network with ReLU activations applies a given affine function over a distribution of inputs. Specifically, as neural networks with ReLU activations are piecewise affine, we derive the probability mass function (PMF) of any affine function being applied.  For the case of inputs distributed as a mixture of Gaussians, we demonstrate that this probability can be computed as the orthant probability corresponding to a Gaussian distribution.
    \item Using the insights from deriving the PMF, we prove that for inputs distributed as a mixture of Gaussians, the outputs are distributed as a mixture of \textit{truncated} Gaussians. Furthermore, we provide an exact expression for the PDF that can be computed as a sum of Gaussian orthant probabilities. 
    \item We propose a sample-free algorithm to identify the most probable activation patterns corresponding to inputs distributed as a mixture of Gaussians, thereby approximating the support of the distribution over affine functions realized by the neural network. 
\end{itemize}

\section{Background}
Let $f(\cdot)$ be a feed-forward neural network that takes as input $\x \in \R^{n_0}$ and produces output $\y \in \R^{n_{L}}$. Each layer $\ell \in \{1,2,\cdots, L\}$ of the network has $n_{\ell}$ neurons, with weights and biases respectively given by the $n_{\ell} \times n_{\ell-1}$ matrix $\W_\ell$ and the vector $\mathbf{b}_{\ell}$ of size ${n}_\ell$. We assume the network employs rectifier activations, $\mathrm{relu}(\cdot)$, that apply the element-wise transformation $\max(0,x)$ to each component of the input vector. Our framework can be readily extended to accommodate other piecewise-linear activations, such as Leaky ReLU, as detailed in \Cref{app:leakyrelu}. Denoting $\h_{\ell}$ as the latent representation computed at layer $\ell$, the network output can be described as:
\begin{subequations}
\begin{align}
\h_{1} &= \W_1 \x + \mathbf{b}_1 \label{eq:ff_nn1}\\
\h_\ell &= \W_\ell \mathrm{relu}(\h_{\ell-1}) + \mathbf{b}_\ell \label{eq:ff_nn2} \\
\y &= \h_{L} \label{eq:ff_nn3} 
\end{align}
\end{subequations}
Feed forward neural networks with ReLU activations are well known to be continuous, piecewise, affine functions \cite{montufar2014number, arora2016understanding, balestriero2018spline, hinz2021analysis}. These networks partition the input space into a collection of non-overlapping, convex polytopes, such that all inputs within the same polytope undergo the same affine transformation. 

\textit{Example}: To clearly illustrate how neural networks partition the input space into polytopes, let us consider the network in~\eqref{eq:ff_nn1}-\eqref{eq:ff_nn3} with $L=2$, meaning the network has a single ReLU nonlinearity. Applying this network on input $\x$ yields output $\y = \W_2\mathrm{relu}\left(\W_1 \x + \mathbf{b}_1\right) + \mathbf{b}_2$. We can equivalently express the output as $\y = \W_2\diag[\z](\W_1\x + \mathbf{b}_1) + \mathbf{b}_2$, where $\z$ is a binary vector that masks the negative entries of $\W_1\x + \mathbf{b}_1$. Two inputs $\x$ and $\x'$ correspond to the same mask $\z$ if and only if $\W_{1}\x + \mathbf{b}_1$ and $\W_{1}\x' + \mathbf{b}_1$ have negative components at the same indices, or equivalently, they lie in the same convex polytope defined by the conditions $(\W_{1}\x + \mathbf{b}_1)_i \leq 0$ or $> 0$ depending on whether entry $i$ of $\z$ is $0$ or $1$, respectively. Since all inputs within this polytope correspond to the same mask $\z$, the ReLU network transforms them using the same affine function.

The previous example generalizes to ReLU networks with $L>2$ layers by considering the binary masks at each layer. Let $\1(\cdot)$ be the elementwise sign function defined as
\begin{align}
    \1(\x)_i &= \begin{cases}
        1 &\text{if $x_i > 0$}\\ 
        0 &\text{otherwise}
    \end{cases}
\end{align}
and let $\z_\ell \in \{0,1\}^{n_{\ell}}$ be the binary vector that takes value 1 only at the indices where $\h_{\ell}$ is positive, i.e., $\z_{\ell} = \1(\h_{\ell})$. Thus, the vector $\z_\ell$ indicates which neurons are active at layer $\ell$. We define the activation pattern by collecting the masks $\z_{\ell}$ at each hidden layer of the network.
\begin{definition}
    The \textit{activation pattern} of the network is defined as $\bzeta = \left[\z_{1}^\top, \dots, \z_{L-1}^\top\right]^\top \in \{0,1\}^{\sum_{\ell=1}^{L-1}n_\ell}$ where each $\z_\ell = \1(\h_\ell)$.
\end{definition}
It can be shown that deep ReLU networks also partition a bounded input space into disjoint convex polytopes \cite{chu2018exact, hanin2019deep}, and within each polytope, all inputs correspond to the same activation pattern $\bzeta$. As a result, inputs from the same polytope are transformed by the same affine function.

\section{ReLU Networks as Random Functions}
\label{sec:relu_networks}
When the inputs $\x$ are random, the activation pattern $\bzeta$ becomes a discrete random variable and induces a discrete distribution over the affine functions realizable by the ReLU network. In \Cref{subsec:affine_distr}, we explicitly characterize the probability mass function (PMF) of $\bzeta$, thereby deriving the distribution over the affine functions associated with each activation pattern. In \Cref{subsec:output_distr}, we then describe how this translates to characterizing the output distribution of the network. We denote a realization of the random variable $\bzeta$ by $\bzeta' \in \{0,1\}^{\sum_{\ell=1}^{L-1}n_\ell}$. The inputs $\x$ are assumed to be distributed as a mixture of multivariate Gaussians with probability density
\begin{align}
    p_{\x}(\x) &= \sum_{k=1}^K \alpha_k \N(\x; \bmu_k, \bSigma_k) \label{eq:x0_gmm}
\end{align}
where $\N(\x; \bmu_k, \bSigma_k)$ is the multivariate normal density with mean $\bmu_k$ and positive definite covariance $\bSigma_k$, and the mixing weights satisfy $\alpha_k > 0, \sum_{k=1}^K \alpha_k = 1$. Gaussian mixture models are universal approximators of probability densities making this setting indeed quite general \cite{ lo1972finite, Goodfellow-et-al-2016}. Additionally, this framework accommodates the non-parametric approach of kernel density estimation using Gaussian kernels.

\subsection{Function distribution}
\label{subsec:affine_distr}
As discussed, each realized activation pattern $\bzeta'$ uniquely identifies a convex polytope in the input space. Therefore, one approach for determining the probability $P(\bzeta = \bzeta')$ is to compute the probability that the input to the network $\x$ lies within the polytope corresponding to $\bzeta'$. Let $\C_{1} = \W_{1}$ and $\mathbf{d}_{1} = \mathbf{b}_{1}$ and, for a given activation pattern $\bzeta' = \left[\z_{1}^\top, \dots, \z_{L-1}^\top\right]^\top$, we recursively define:
\begin{align}
    \C_{\ell} &= \W_{\ell}\diag[\z_{\ell-1}]\C_{\ell-1} \label{eq:C}\\
    \mathbf{d}_{\ell}&=  \W_{\ell}\diag[\z_{\ell-1}]\mathbf{d}_{\ell-1} + \mathbf{b}_{\ell} \label{eq:d}
\end{align}
Hence, if input $\x$ induces activation pattern $\bzeta'$, the network computes intermediate representations as $\h_\ell = \C_\ell\x + \mathbf{d}_\ell$ and output $f(\x) = \h_L = \C_L\x + \mathbf{d}_L$. We collect the matrices $\C_{1},\dots,\C_{\ell}$ and vectors $\mathbf{d}_1,\dots,\mathbf{d}_{\ell}$ into the following block matrix and vector respectively:
\begin{align}
    \lC^{(\ell)}_{\bzeta'} = \begin{bmatrix}
        \C_{1}\\
        \C_{2}\\
        \vdots\\
        \C_{\ell}
    \end{bmatrix},\;\;
    \ld^{(\ell)}_{\bzeta'} = \begin{bmatrix}
        \mathbf{d}_{1}\\
        \mathbf{d}_{2}\\
        \vdots\\
        \mathbf{d}_{\ell}
    \end{bmatrix}
    \label{eq:Cd}
\end{align}

Using the block matrix and vector in~\eqref{eq:Cd} above, we can give a form for the convex polytope containing all inputs $x \in \R^{n_0}$ that are transformed by the same affine function. 

\begin{lemma}
    An input $\x$ induces activation pattern $\bzeta'$ if and only if it is in the convex polytope:
    \begin{align}
        \mathcal{K}(\bzeta')\!&=\!\left\{\!\x\!\in\!\R^{n_0}\!:\!\begin{cases}
            \left[\lC_{\bzeta'}^{(L-1)}\x - \ld_{\bzeta'}^{(L-1)}\right]_i\!\leq\!0 \text{ if } \zeta'_i\!=\!0\\[5pt]
             \left[\lC_{\bzeta'}^{(L-1)}\x - \ld_{\bzeta'}^{(L-1)}\right]_i\!>\!0 \text{ if } \zeta'_i\!=\!1
        \end{cases}\hspace{-1em}\right\}
        \label{eq:polytope}
    \end{align}
    \label{lemma:polytope}
\end{lemma}
The inequality constraints in~\eqref{eq:polytope} ensure that the sign of each hidden representation $\h_{\ell}$ correctly identify the activation pattern $\bzeta'$. The proof of \Cref{lemma:polytope} follows immediately from~\eqref{eq:polytope} since, if $\x \in \mathcal{K}(\bzeta')$ then it satisfies all linear inequality constraints that guarantee activation pattern $\bzeta'$ is induced. If $\x \notin \mathcal{K}(\bzeta')$, then there must exist a non-empty set of indices $\mathcal{I} = \{1, 2, \dots, \sum_{\ell=1}^{L-1} n_{\ell}\}$ at which the linear inequality constraints are not satisfied, and $\x$ then corresponds to the activation pattern defined by flipping the bits of $\bzeta'$ at indices $\mathcal{I}$. 

By \Cref{lemma:polytope}, an activation pattern $\bzeta'$ uniquely identifies an activation pattern in the input space. Thus, the PMF of $\bzeta$ can be expressed as:
\begin{align}
    P(\bzeta = \bzeta') &= \int_{\mathcal{K}(\bzeta')} p_{\x}(\x) d\x
    \label{eq:pmf_convex_polytope}
\end{align}
The challenge in numerically computing the integral in~\eqref{eq:pmf_convex_polytope} is that region of integration is non-rectangular. Hence, computation typically requires Markov Chain Monte Carlo (MCMC), methods, such as those described in \cite{belisle1993hit, chen2018fast}, to sample over the polytope. We show that it is possible to express~\eqref{eq:pmf_convex_polytope} as a Gaussian integral over a rectangular region, namely an orthant, thereby simplifying computation.

\begin{definition}
    The orthant associated with binary vector $\bzeta \in \{0,1\}^n$ is defined as $\mathcal{O}(\bzeta) = \{\x \in \R^n : \1(\x) = \bzeta\}$.
\end{definition}

The orthant associated with a binary vector $\bzeta$ contains vectors that are positive at indices where $\bzeta$ is 1. The next result demonstrates how the integral in~\eqref{eq:pmf_convex_polytope} can be transformed to a Gaussian integral over an orthant.

\begin{proposition}
    Let $\x \in \R^{n_0}$ be distributed as a mixture of Gaussians, with probability density in~\eqref{eq:x0_gmm} and let:
    \begin{align}
    \widetilde{\bmu}_{k,\bzeta'} &= \lC_{\bzeta'}^{(L-1)}\bmu_k + \ld_{\bzeta'}^{(L-1)}\\
    \widetilde{\bSigma}_{k,\bzeta'} &= \lC_{\bzeta'}^{(L-1)}\bSigma_k{\lC_{\bzeta'}^{(L-1)}}^\top
    \end{align}
    If $\widetilde{\bSigma}_{k,\bzeta'}$ is invertible for all $k$, then:
    \begin{align}
        P\left(\bzeta = \bzeta'\right) = \sum_{k=1}^K \alpha_k \int_{\mathcal{O}(\bzeta')} \N(\x; \widetilde{\bmu}_{k, \bzeta'}, \widetilde{\bSigma}_{k, \bzeta'}) d\x
        \label{eq:pmf}
    \end{align}
    \label{prop:pmf}
\end{proposition}

\begin{proof}
The probability $P(\bzeta = \bzeta')$ can be expressed in terms of the latent representations as:
    \begin{align*}
    P(\bzeta = \bzeta') &= P[(\h_{L-1} \in \mathcal{O}(\z_{L-1})) \cap (\h_{L-2} \in \mathcal{O}(\z_{L-2})) \cap \dots \cap (\h_{1} \in \mathcal{O}(\z_{1}))]\\
    &= P\left[\lC^{(L-1)}_{\bzeta'}\x + \ld^{(L-1)}_{\bzeta'} \in \mathcal{O}(\bzeta')\right]
\end{align*}
Since $\x$ is distributed as a mixture of Gaussians with density~\eqref{eq:x0_gmm}, $\lC^{(L-1)}_{\bzeta'}\x + \ld^{(L-1)}_{\bzeta'}$ is distributed as a mixture of Gaussians. The density for $\lC^{(L-1)}_{\bzeta'}\x + \ld^{(L-1)}_{\bzeta'}$ is then given by $\sum_{k=1}^K \alpha_k \N(\x; \bmu_{k,\bzeta'}, \bSigma_{k,\bzeta'})$
where:
\begin{align*}
    \bmu_{k,\bzeta'} &= \lC^{(L-1)}_{\bzeta'}\bmu_k + \ld^{(L-1)}_{\bzeta'}\\
    \bSigma_{k,\bzeta'} &= \lC^{(L-1)}_{\bzeta'}\bSigma_k{\lC^{(L-1)}_{\bzeta'}}^\top
\end{align*}
Computing $P\left[\lC^{(L-1)}_{\bzeta'}\x + \ld^{(L-1)}_{\bzeta'} \in \mathcal{O}(\bzeta')\right]$ is then equivalent to integrating the density of $\lC_{\bzeta'}^{(L-1)}\x + \ld^{(L-1)}_{\bzeta'}$ over the region $\mathcal{O}(\bzeta')$.
\end{proof}

The proof of \Cref{prop:pmf} rewrites the integral in~\eqref{eq:pmf_convex_polytope} as the probability of a set of jointly Gaussian variables. We make a few remarks regarding the result. First,~\eqref{eq:pmf} shows that the PMF of $\bzeta$ can be expressed as a sum of Gaussian probabilities over orthants, which have been extensively studied \cite{Stack1962, owen2004orthant, genz2009computation, ridgway2016computation, azzimonti2018estimating}. Second, the number of variables in the integral is $\sum_{\ell=1}^{L-1} n_\ell$, where we recall $n_{\ell}$ is the number of neurons at layer $\ell$ in the network. Thus, the integral in~\eqref{eq:pmf} has fewer variables than that in~\eqref{eq:pmf_convex_polytope} when the total number of hidden neurons is less than the dimensionality of the input. When  $\sum_{\ell=1}^{L-1} n_\ell$ is larger than the input dimensionality, the resulting covariances $\lC_{\bzeta_i}\bSigma_k\lC_{\bzeta_i}^\top$ are rank deficient and a suitable transformation can be applied, such as the generalized Cholesky factor or the eigenvectors of the covariance matrix corresponding to its nonzero eigenvalues \cite{rao1973linear, genz2009computation}. Furthermore, we observe that in such scenarios, diagonal approximation of the covariances performs reasonably well in practice. 

\subsection{Output distribution}
\label{subsec:output_distr}
Our method for deriving the PMF of $\bzeta$ in \Cref{prop:pmf} can also be used to characterize the distribution over the neural network outputs. In order to determine the probability density $p_{\y}(\y)$ of the output vector $\y = f(\x)$, we can consider marginalizing over the joint distribution of neural network inputs and outputs
\begin{align}
    p_{\y}(\y) &= \int_{\x} p_{\x,\y}(\x,\y) d\x
\end{align}
and leverage the fact that $f(\cdot)$ is a piecewise affine function to rewrite the integral. Let $\mathrm{supp}(\bzeta) = \{\bzeta' : P(\bzeta = \bzeta') > 0\}$ denote the support of the random activation pattern $\bzeta$, and hence $|\mathrm{supp}(\bzeta)|$ is equivalent to the number of convex polytope partitions in the input space that are occupied by the inputs with nonzero probability. The regions $\mathcal{K}(\bzeta')$ are disjoint hence the density of $\y$ can be expressed as follows:
\begin{align}
    p_\y(\y) &= \sum_{\bzeta' \in \mathrm{supp}(\bzeta)} \int_{\mathcal{K}(\bzeta')} p_{\x, \y}(\x, \y) d\x \label{eq:pdfy_polytope}
\end{align}
By expanding the rightmost term in~\eqref{eq:pdfy_polytope} using Bayes' theorem, we obtain the following result.
\begin{proposition}
    Let $\x$ be distributed as a mixture of Gaussians, and let $f(\x)$ be piecewise affine. Then $\y = f(\x)$ is distributed as a mixture of \underline{truncated} Gaussians.
    \label{prop:trunc_gaussians}
\end{proposition}
\begin{proof}
We consider marginalizing the joint distribution of inputs and outputs as follows:
\begin{align*}
    p_{\y}({\y}) &= \int_{\x} p_{\x,{\y}}(\x,\y) d\x
\end{align*}
$f(\cdot)$ is a piecewise affine function. Let $\X_i,\;i\in\{1,2,\dots,M\}$ be the regions on which $f$ behaves as a consistent affine function. The regions $\X_i$ satisfy $\X_i \cap \X_j = \emptyset, i \neq j$ and $\cup_{i} \X_i = \R^{n}$. Hence we can express the integral above as: 
\begin{align*}
    p_{\y}(\y) &= \int_{\x} p_{\x,\y}(\x,\y) d\x\\
    &= \sum_{i=1}^M \int_{\X_i} p_{\x,\y}(\x,\y) d\x\\
    &= \sum_{i=1}^M p(\y | \x \in \X_i) P(\x \in \X_i) \label{eq:py_trunc_mix}
\end{align*}
The network output can be written as $\y = \C_i \x + \mathbf{d}_i,\;\forall \x \in \mathcal{X}_i$. The distribution of $\y$ conditioned on $\x \in \X_i$ therefore has the following density:
\begin{align}
    p(\y | \x \in \X_i) &= p(\mathbf{C}_i \x + \mathbf{d}_i | \x \in \X_i)\\
    &= \begin{cases}
        \frac{1}{Z}\sum_{k=1}^K \alpha_k \N(\y; \mathbf{C}_i \bmu_k + \mathbf{d}_i, \mathbf{C}_i \bSigma_{k} \mathbf{C}_i^\top) \;\;\text{ if $\y \in \Y_i$} \\
        0 \text{ otherwise}        
    \end{cases} \label{eq:trunc_gmm}
\end{align}
where $\Y_i$ is the image of the region $\X_i$ under the affine transformation $\mathbf{C}_i \x + \mathbf{d}_i$ and the normalization constant $Z$ is given by $Z = \int_{\y \in \Y_i} \sum_{k=1}^K \alpha_k \N(\y; \mathbf{C}_i \bmu_k + \mathbf{d}_i, \mathbf{C}_i \bSigma_{k} \mathbf{C}_i^\top) d\y$. We can equivalently express~\eqref{eq:trunc_gmm} as:
\begin{align}
    p(\y | \x \in \X_i) &= \frac{1}{Z}\sum_{k=1}^K\alpha_k\mathcal{T}(\y) \label{eq:cond_prob}\\
    \mathcal{T}(\y) &= \begin{cases}
        \N(\y; \mathbf{C}_i \bmu_k + \mathbf{d}_i, \mathbf{C}_i \bSigma_{k} \mathbf{C}_i^\top) \;\;\text{ if $\y \in \Y_i$} \\
        0 \text{ otherwise}        
    \end{cases}
\end{align}
Substituting~\eqref{eq:cond_prob} back into~\eqref{eq:py_trunc_mix} shows that $p_{\y}(\y)$ is the density of the mixture of $MK$ truncated Gaussians. In the case where $\mathbf{C}_i \bSigma_k {\mathbf{C}_i}^\top$ is not invertible, the density of $\y$ exists on a subspace and can be expressed in terms of the degenerate Gaussian density \cite{rao1973linear} that uses the pseudoinverse and pseudodeterminant of the covariance.
\end{proof}

A common assumption in works related to neural network uncertainty propagation is that the intermediate representations, $\h_\ell$, obey a Gaussian distribution \cite{abdelaziz2015uncertainty, gast2018lightweight}. However, \Cref{prop:trunc_gaussians} demonstrates that for neural networks with piecewise linear activations and arbitrary weight and bias parameters, the intermediate representations are actually distributed as a mixture of truncated Gaussians. As in \Cref{prop:pmf}, we can transform the expression for $p_\y(\y)$ in~\eqref{eq:pdfy_polytope} into a sum of Gaussian integrals over orthants.

\begin{proposition}
Let $\x$ be distributed as a mixture of Gaussians, with probability density in~\eqref{eq:x0_gmm} and let:
\begin{align}
\widetilde{\bmu}_{k,\bzeta'} &= \lC_{\bzeta'}^{(L)}\bmu_k + \ld_{\bzeta'}^{(L)}\\
\widetilde{\bSigma}_{k,\bzeta'} &= \lC_{\bzeta'}^{(L)}\bSigma_k{\lC_{\bzeta'}^{(L)}}^\top
\end{align}
If $\widetilde{\bSigma}_{k,\bzeta'}$ is invertible for all $k$, then:
\begin{align}
    p_{\y}(\y) &= \sum_{\bzeta' \in \mathrm{supp}(\bzeta)}\sum_{k=1}^K \alpha_k \int_{\mathcal{O}(\bzeta')} \N\left(\begin{bmatrix}
        \x\\
        \y
    \end{bmatrix}; \widetilde{\bmu}_{k, \bzeta'}, \widetilde{\bSigma}_{k, \bzeta'}\right) d\x
    \label{eq:pdf_y}
\end{align}
\label{prop:pdf}
\end{proposition}
\begin{proof}
    Starting again with the joint distribution over the inputs and outputs, we can consider the CDF of $\y = f(\x)$ as:
\begin{align*}
   P(\y < \boldsymbol{\phi}) &=  \int_{-\boldsymbol{\infty}}^{\boldsymbol{\phi}}\int_{\x} p_{\x,\y}(\x, \y) d\x
\end{align*}
Let $\X_i,\;i\in\{1,2,\dots,M\}$ be the regions on which $f$ behaves as a consistent affine function. The regions $\X_i$ are disjoint and span the input space. so we can expand the integral above as: 
\begin{align*}
        P(\y < \boldsymbol{\phi}) &= \sum_{i=1}^{M} \int_{-\boldsymbol{\infty}}^{\boldsymbol{\phi}} \int_{\X_i} p_{\x, \y}(\x, \y) d\x d\y
\end{align*}
where the order of sums and integrals can be exchanged by the Fubini-Tonelli theorem. In the expression above we note that:
\begin{align*}
    \int_{-\boldsymbol{\infty}}^{\boldsymbol{\phi}} \int_{\X_i} p_{\x, \y}(\x, \y) d\x d\y &= P[(\y < \boldsymbol{\phi}) \cap (\x \in \X_i)] \\
    &= P\left[(\C_{L}\x + \mathbf{d}_{L} < \boldsymbol{\phi}) \cap \left(\lC_{\bzeta_i}^{(L-1)}\x + \ld^{(L-1)}_{\bzeta_i} \in \mathcal{O}(\bzeta)\right)\right]
\end{align*}
where $\bzeta_i$ is the activation pattern corresponding to region $\mathcal{X}_i$, and $\C_{L}, \mathbf{d}_{L}$, as defined in~\eqref{eq:C}-\eqref{eq:d}, depend on $\bzeta_i$. Augmenting $\C_{L}, \mathbf{d}_{L}$ to $\lC_{\bzeta_i}^{(L-1)}, \ld^{(L-1)}_{\bzeta_i}$ respectively yields the block matrix and block vector $\lC_{\bzeta_i}^{(L)},\;\ld^{(L)}_{\bzeta_i}$. Defining $\widetilde{\bmu}_{k, \bzeta'}$ and $ \widetilde{\bSigma}_{k, \bzeta'}$ as in \Cref{prop:pdf} and assuming $\widetilde{\bSigma}_{k, \bzeta'}$ is invertible, we can express the previous probability as:
\begin{align*}
    &P\left[(\C_{L}\x + \mathbf{d}_{L} < \boldsymbol{\phi}) \cap \left(\lC_{\bzeta_i}^{(L-1)}\x + \ld^{(L-1)}_{\bzeta_i} \in \mathcal{O}(\bzeta)\right)\right] =\\ &\qquad\qquad\qquad\qquad\qquad\sum_{k=1}^K \alpha_k \int_{-\boldsymbol{\infty}}^{\boldsymbol{\phi}} \int_{\mathcal{O}(\bzeta')}  
    \N\left(\begin{bmatrix}
        \x\\
        \y
    \end{bmatrix}; \widetilde{\bmu}_{k, \bzeta'}, \widetilde{\bSigma}_{k, \bzeta'}\right) d\x
\end{align*} The density in~\eqref{eq:pdf_y} then follows from the above by the fundamental theorem of calculus.
\end{proof}
In \Cref{prop:pdf}, $\widetilde{\bmu}_{k,\bzeta'}$ and $\widetilde{\bSigma}_{k,\bzeta'}$ are defined using $\lC_{\bzeta'}^{(L)}$ and $\ld_{\bzeta'}^{(L)}$, whereas in \Cref{prop:pmf} they were defined using $\lC_{\bzeta'}^{(L-1)}$ and $\ld_{\bzeta'}^{(L-1)}$. If the covariances $\widetilde{\bSigma}_{k,\bzeta'}$ are approximated as diagonal matrices, the integral can be efficiently evaluated as a product of terms involving the error function. Additionally, approximating each covariance as the sum of a diagonal and low rank matrix can significantly reduce the dimensionality of the integration to the rank of the low rank approximating matrix \cite{genz2009computation}. 

As in~\eqref{eq:pmf}, the expression derived in~\eqref{eq:pdf_y} demonstrates that the density of $\y$ can be expressed as a sum of Gaussian orthant probabilities. While direct computation of the density of $\y$ is possible, it requires identifying the appropriate region for integration, namely the intersection of the convex polytope $\mathcal{K}(\bzeta')$ and the preimage of $f(\y)$ \cite{krapf2024piecewise}. Hence, the orthant formulation in~\eqref{eq:pdf_y} provides an attractive alternative. Moreover, the formulation presented in \Cref{prop:pdf} reduces the dimensionality of the integral when $\sum_{\ell=1}^L n_{\ell} < n_0$. 

\section{Approximating the support of $\bzeta$}
\label{sec:approximation}
The expression for the density of $\y$ given in~\eqref{eq:pdf_y} depends on $\mathrm{supp}(\bzeta)$, the set of activation patterns that occur with nonzero probability when the inputs are distributed as a mixture of Gaussians. Therefore, $|\mathrm{supp}(\bzeta)|$ is bounded above by the total number of affine regions into which the network partitions the input space. Determining the total number of affine regions, either exactly for shallow networks or using bounds for deeper networks, has been a subject of ongoing research \cite{pascanu2013number, montufar2014number, telgarsky2015representation, raghu2017expressive, serra2018bounding, hu2020analysis, piwek2023exact, goujon2024number}. 

Generally, the maximum number of affine regions in deep ReLU networks scales polynomially with the size of the hidden layers, and exponentially with respect to the number of layers. However, in practice, it has been observed that ReLU networks learn functions with significantly fewer regions than the theoretical limit \cite{hanin2019complexity, hanin2019deep, serra2020empirical}, thereby suggesting that the parameters allowing a ReLU network to achieve the theoretical upper bound on the number of affine regions occupy only a small fraction of the global parameter space \cite{patel2024local}. The question arises: given a distribution of inputs, how can we identify a subset of neural network activation patterns, and thus affine functions, that occur with high probability? 

For neural networks with a large number of hidden neurons, estimating the joint probability distribution over the activation patterns from samples is intractable due to the curse of dimensionality. We provide a sample-free method in \Cref{alg:mainalg}, which relies on the subroutines outlined in \Cref{alg:getparameters} and \Cref{alg:getpatterns}. We describe the intuition of the procedure in this section, which leverages recent observations that some neurons exhibit more randomness in their activations, while others behave effectively deterministically \cite{whitaker2023synaptic}. Hence, we can effectively prune neurons at each layer that exhibit low entropy by computing the probabilities $P(\h_{\ell, j} > 0)$ where $\h_{\ell, j}$ denotes entry $j$ of the feature $\h_\ell$. These probabilities represent the likelihood that neuron $j$ at layer $\ell$ will be active. If the neuron is inactive or active with high probability, we can respectively set the corresponding entry $j$ in $\z_{\ell}$ to 0 or 1. Repeating this procedure across all layers in the network enables us to significantly reduce the number of activation patterns considered. 

We examine the procedure in detail for the network in~\eqref{eq:ff_nn1}-\eqref{eq:ff_nn3} with $L=4$, meaning the network has three ReLU nonlinearities. Letting the input to the network $\x$ be distributed as a mixture of Gaussians with parameters $\{\alpha_k, \bmu_k, \bSigma_k\}_{k=1}^K$, we compute the parameters of the distribution corresponding to $\h_{1}$ as $\{\alpha_k, \W_{1}\bmu_k + \mathbf{b}_1, \W_{1}\bSigma_k\W_{1}^\top\}_{k=1}^K$. Next, we compute the marginal probabilities that neuron $j$ is active in layer 1, namely the probabilities $P(h_{1,j} > 0)$. These marginal probabilities can be efficiently computed in terms of univariate Gaussian CDF $\Phi$ as:
\begin{align}
    P(h_{1,j} > 0) &= \sum_{k} \alpha_k\Phi\left(\frac{\mu_k}{\sigma_{k,j}}\right) \label{eq:marginal_prob}
\end{align}
where $\mu_{k,j}$ and $\sigma_{k,j}^2$ are the mean and variance of $h_{1,j}$. We use the computed marginal probabilities to identify the neurons with entropy larger than predefined threshold. For all other neurons, we deterministically set them to be 1 or 0 depending on the value of $P(h_{1,j} > 0)$. This approach, as outlined
in \Cref{alg:getpatterns}, reduces the number of activation patterns considered at the first layer from $2^{n_1}$ to $2^{\tau_1}$, where $\tau_1$ is the number of neurons with entropy larger than $\tau$. 

Collecting the  $2^{\tau_1}$ activation patterns at layer $1$ into the set $\mathcal{Z}_1$, we can identify the corresponding activation patterns at layer 2 by conditioning on the activation patterns in $\mathcal{Z}_1$. For each $\z_1 \in \mathcal{Z}_1$, we compute the parameters $\{\alpha_k, \W_{2}\diag[\z_{1}](\W_{1}\bmu_k + \mathbf{b}_1) + \mathbf{b}_2, \W_{2}\diag[\z_{1}]\W_{1}\bSigma_k\W_{1}^\top\diag[\z_{1}]^\top\W_{2}^\top\}_{k=1}^K$ of the conditional distribution $\h_{1} | \z_{1}$ using \Cref{alg:getparameters}. As before, we then compute the marginal probabilities and identify the neurons with highest entropy, considering only the corresponding activation patterns.

By conditioning on the activation patterns from the earlier layers, we can continue identifying the corresponding activation patterns at the subsequent layers of the network, and approximately identify a highly probable subset of activation patterns. The algorithm presented in \Cref{alg:getpatterns} can be viewed as a branch-and-bound method for exploring the activation pattern state space. At each layer of the search, the bounding function is given by:
\begin{align*}
    P(\mathbf{z}_\ell) \leq \min_{i \in \{1, \dots, n_{\ell}\}} P(z_{\ell,i})
\end{align*}
Branches are pruned if the minimum marginal probability across all neurons $z_{\ell, i}$ is sufficiently small, indicating that certain neurons are effectively deterministic.

\begin{algorithm}[htpb]
\small
\caption{Estimate support}
\label{alg:mainalg}
\begin{algorithmic}[1]
\REQUIRE Number of layers $L$, input distribution parameters $\{\alpha_k, \bmu_k, \bSigma_k\}_{k=1}^K$,  thresholds $\{\tau_\ell\}_{\ell = 1}^{L-1}$
\ENSURE Representative activation patterns in curList
\STATE Initialize $\text{curList} = [\;]$
\STATE $\mathcal{Z}$ = $\mathit{GetPatterns}\left(\{\alpha_k, \W_1\bmu_k + \mathbf{b}_1, \W_1\bSigma_k\W_1^\top\}_{k=1}^K, \tau_1\right)$
\FOR{$\z \in \mathcal{Z}$}
    \STATE Append list $[\z]$ to $\text{curList}$
\ENDFOR
\STATE Initialize $\text{nextList} = [\;]$
\FOR{$\ell = 2$ to $L - 1$}
    \FOR{each $\text{subList}$ in \text{curList}}
        \STATE $\C, \mathbf{d} = \mathit{GetParams}(\text{subList})$
        \STATE $\mathcal{Z} = \mathit{GetPatterns}\left(\{\alpha_k, \C\bmu_k + \mathbf{d}, \C\bSigma_k\C^\top\}_{k=1}^K, \tau_\ell\right)$
        \FOR{$\z \in \mathcal{Z}$}
            \STATE Append $\z$ to a copy of subList
            \STATE Append resulting list from prior step to nextList
        \ENDFOR
    \ENDFOR
    \STATE $\text{curList} \gets \text{nextList}$
    \STATE $\text{nextList} \gets [\;]$
\ENDFOR
\end{algorithmic}
\end{algorithm}

\begin{algorithm}
\caption{GetParams}
\small
\label{alg:getparameters}
\begin{algorithmic}[1]
\REQUIRE Neural network parameters $\{\W_{\ell}, \mathbf{b}_{\ell}\}_{\ell=1}^{P}$ and activation patterns $\{\z_{\ell}\}_{\ell=1}^{P-1}$ up to layer $P$
\ENSURE $\C = \C_{P}, \mathbf{d}=\mathbf{d}_P$ with $\C_{P}, \mathbf{d}_P$ defined as in~\eqref{eq:C},\eqref{eq:d}
\STATE Initialize $\C = \W_1$
\STATE Initialize $\mathbf{d} = \mathbf{b}_1$ 
\FOR{$\ell = 2$ to $P$}
    \STATE Update $\C \gets \W_{\ell} \diag[\z_{\ell-1}] \C$
    \STATE Update $\mathbf{d} \gets \W_{\ell}\diag[\z_{\ell-1}]\mathbf{d} + \mathbf{b}_{\ell}$
\ENDFOR
\end{algorithmic}
\end{algorithm}

\begin{algorithm}[htpb]
\caption{GetPatterns}
\small
\label{alg:getpatterns}
\begin{algorithmic}[1]
\REQUIRE Distribution parameters $\{\widetilde{\alpha}_k, \widetilde{\bmu}_k, \widetilde{\bSigma}_k\}_{k=1}^K$, threshold $\tau$
\ENSURE Activation patterns $\mathcal{Z}$
\STATE Initialize $\mathbf{s}_k$ as vector of diagonal entries of $\widetilde{\bSigma}_k,\;\forall k$
\STATE Compute $\mathbf{p}$ with $p_i\!=\!\sum_{k=1}^K\!\alpha_k\Phi\left(\widetilde{\mu}_{k,i}/{s_{k,i}}\right)$
\STATE Find $\mathcal{I}$: set of indices of $\mathbf{p}$ with entropy less than $\tau$
\FOR{$i \in \mathcal{I}$}
\STATE $v_i = \begin{cases}
    1 \text{ if } |1 - p_i| < |p_i|\\
    0 \text{ otherwise}
\end{cases}$
\ENDFOR
\STATE $\mathcal{Z} = \{\z \in \{0,1\}^{\mathrm{dim}(\mathbf{p})} : z_i = v_i,\; \forall i \in \mathcal{I}\}$
\end{algorithmic}
\end{algorithm}

\newpage 
\section{Experiments}
\subsection{Affine Function Distribution}
We first use the moons dataset to validate our expressions in~\eqref{eq:pmf} and~\eqref{eq:pdf_y} corresponding respectively to the PMF of the activation pattern and PDF of the outputs of the neural network. We train the neural network described in~\eqref{eq:ff_nn1}-\eqref{eq:ff_nn3}, with $L=4$ and $n_1 = n_2 = n_3 = 4$ neurons at each hidden layer. The number of neurons is deliberately small in this experiment to permit exhaustive enumeration of all activation patterns. We provide the optimization parameters of the training procedure in \Cref{app:experiment_configurations}. 

We start by considering the impact of Gaussian noise on the neural network by randomly selecting 2 test samples, $\x_1, \x_2$. For each test sample $\x_i$, we numerically compute the PMF of the network's activation pattern and the CDF of the network's output corresponding to the isotropic Gaussian input distribution $\N(\x_i, \sigma^2 \mathbf{I})$. The results are shown in \Cref{fig:moons_isotropic}, where the lines in blue indicate the probabilities computed by numerically integrating ~\eqref{eq:pmf} and~\eqref{eq:pdf_y}, and the red lines indicate the probabilities estimated using 1 million Monte Carlo trials. In the right column of \Cref{fig:moons_isotropic}, we observe jump discontinuities in the CDF of the neural network output $y$. These discontinuities indicate that $y$, which by \Cref{prop:trunc_gaussians} is distributed as mixture of truncated Gaussians, is truncated over disjoint intervals. Overall, we observe excellent agreement between the numerically computed probabilities and those obtained via Monte Carlo simulation.

\begin{figure*}[htpb]
    \centering
    \begin{subfigure}[c]{0.32\textwidth}
        \includegraphics[width=\linewidth]{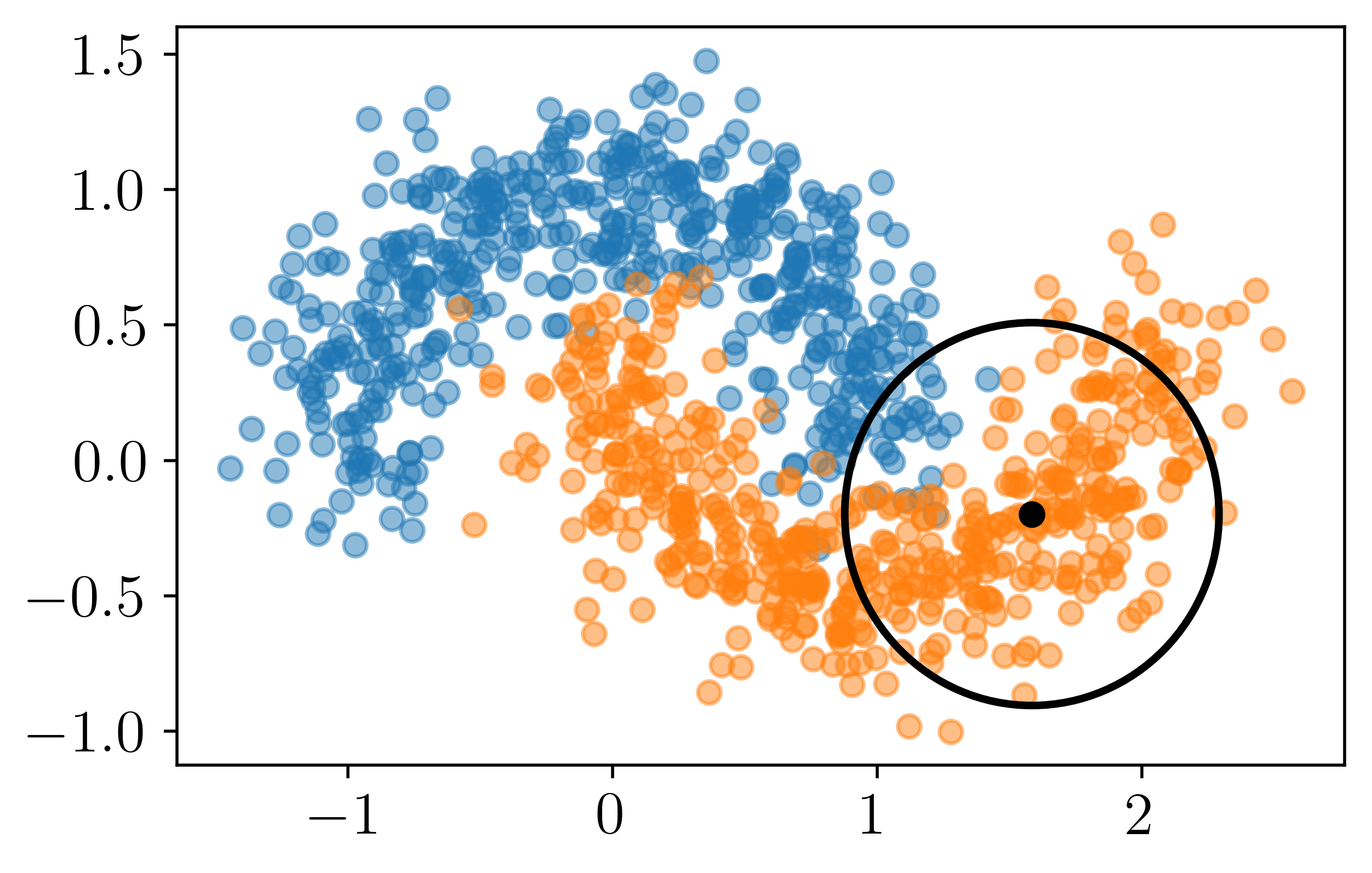}
    \end{subfigure}
    \hfill
    \begin{subfigure}[c]{0.32\textwidth}
        \includegraphics[width=\linewidth]{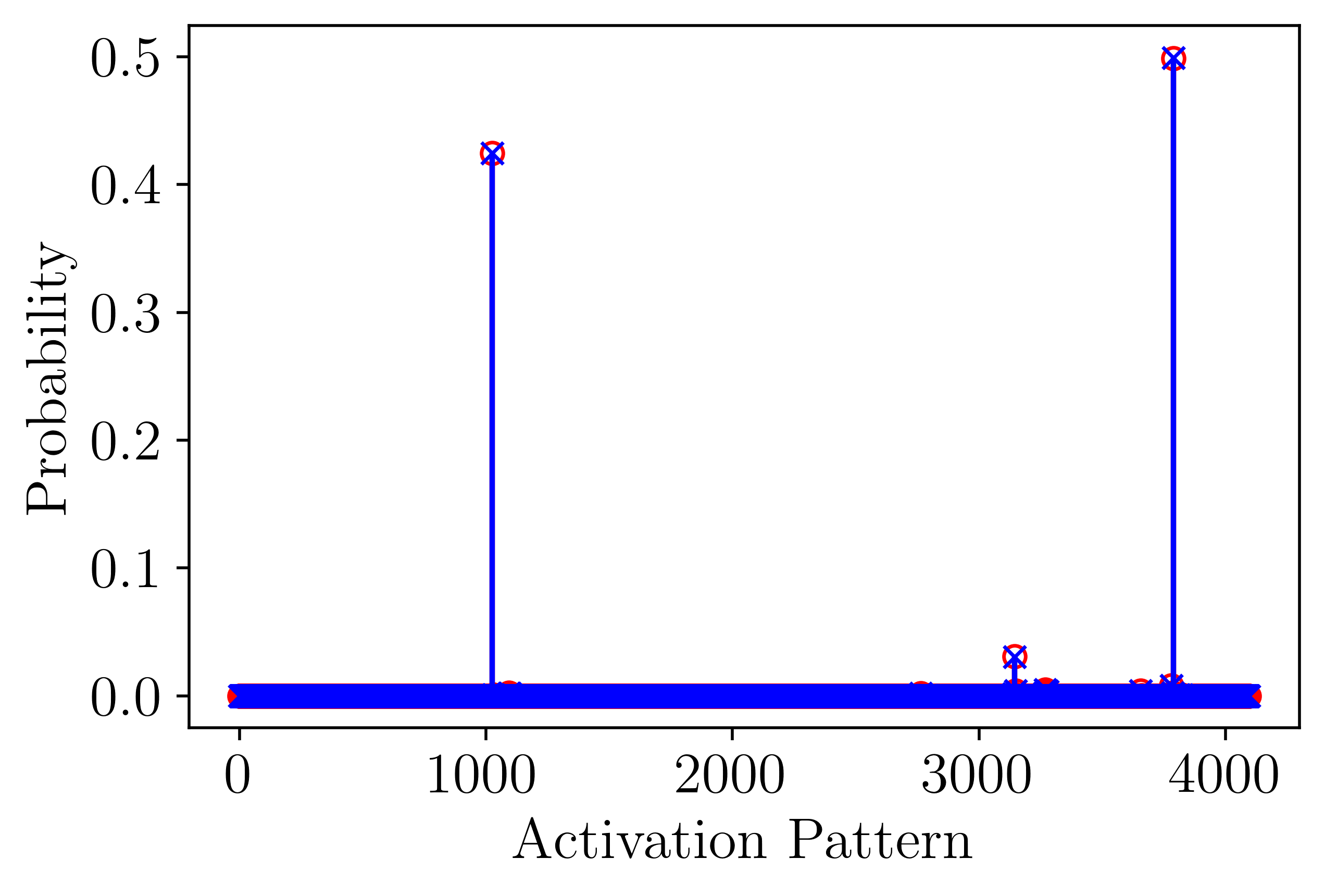} 
    \end{subfigure}
    \hfill
    \begin{subfigure}[c]{0.32\textwidth}
        \includegraphics[width=\linewidth]{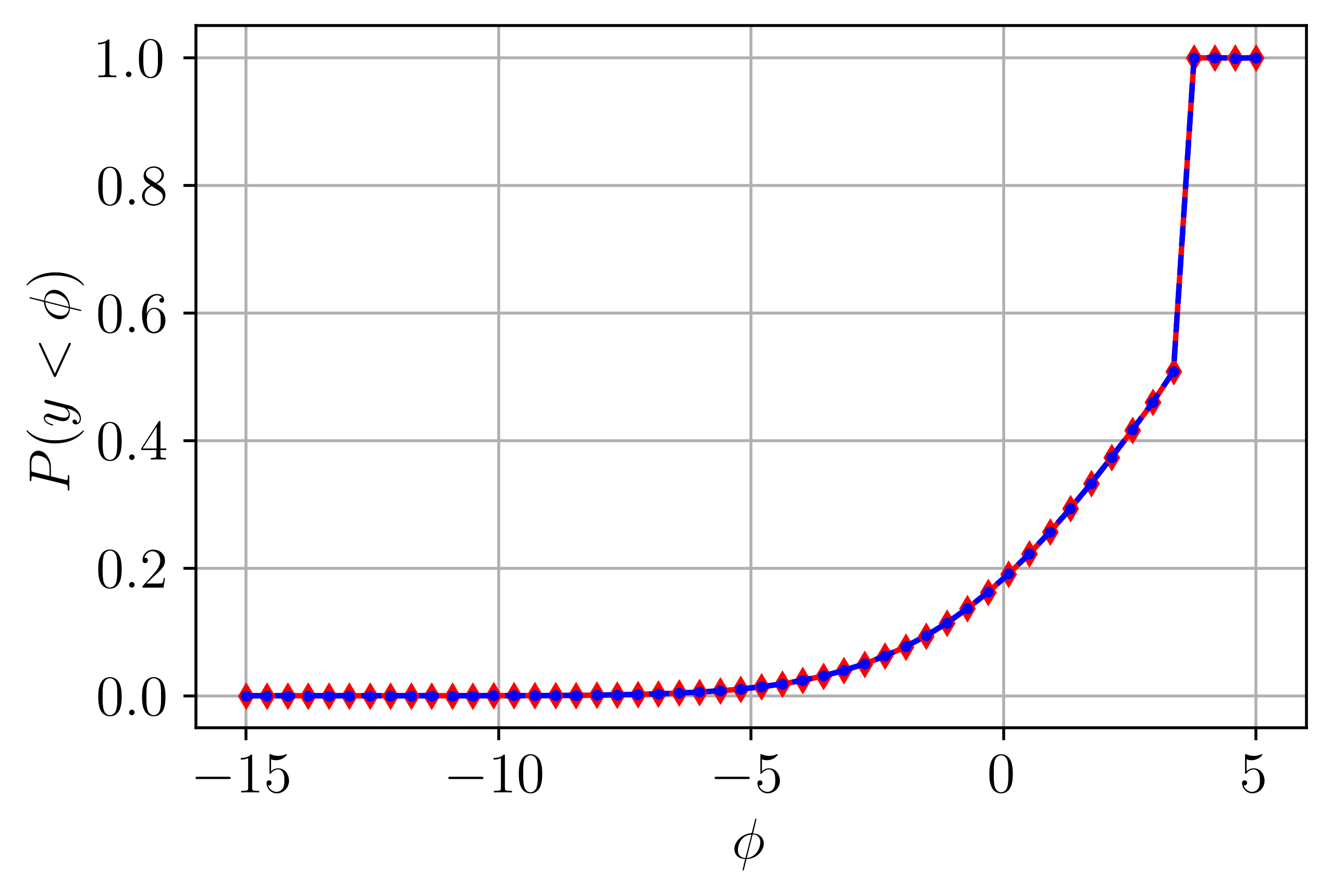} 
    \end{subfigure}
    {\small (a) $\x_1, \sigma^2 = 0.5$ (Blue lines are numerically computed, red lines are Monte Carlo estimates)}\\ \vspace{0.5em}
    
    \begin{subfigure}[c]{0.32\textwidth}
        \includegraphics[width=\linewidth]{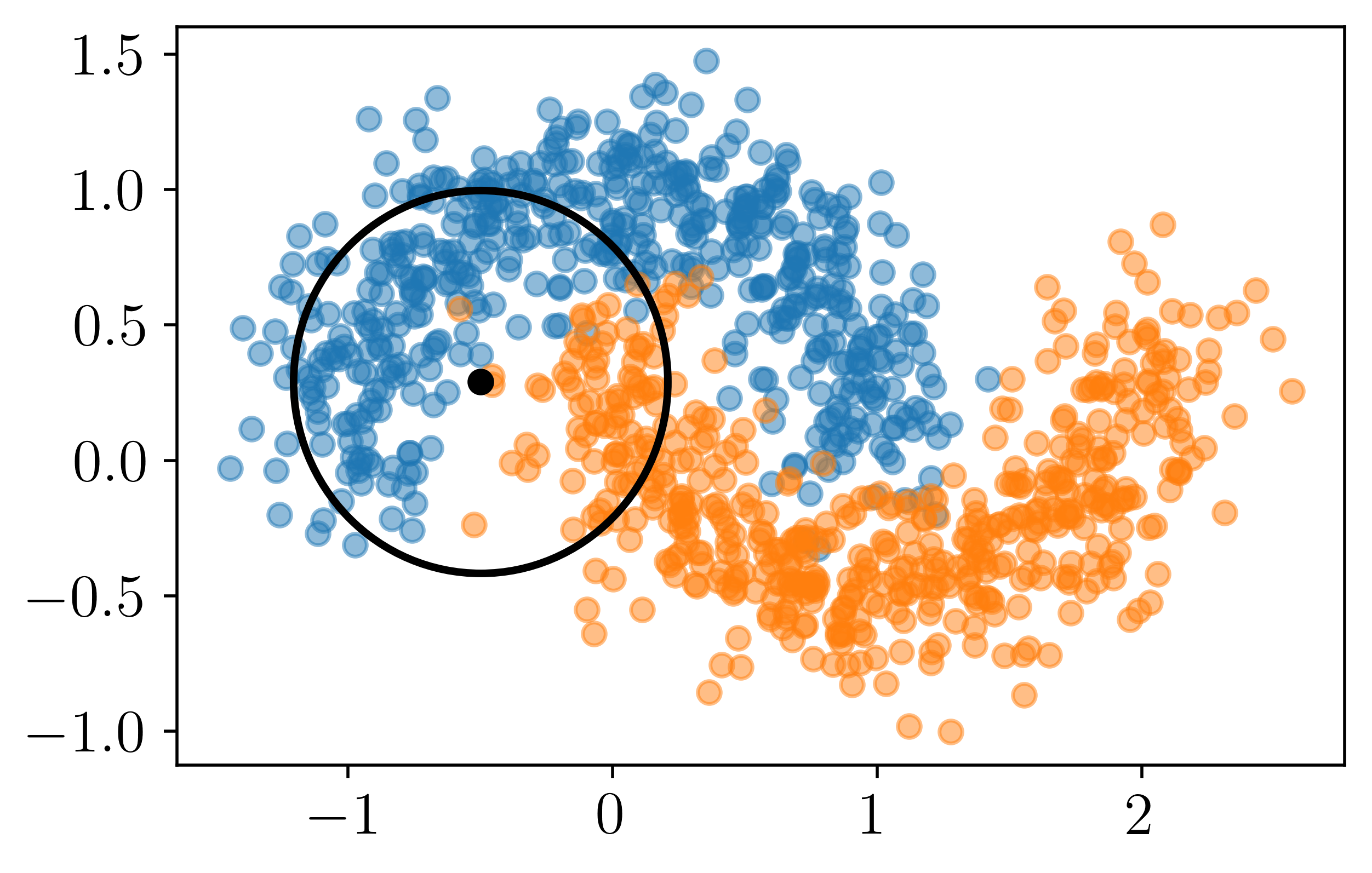} 
    \end{subfigure}
    \hfill
    \begin{subfigure}[c]{0.32\textwidth}
        \includegraphics[width=\linewidth]{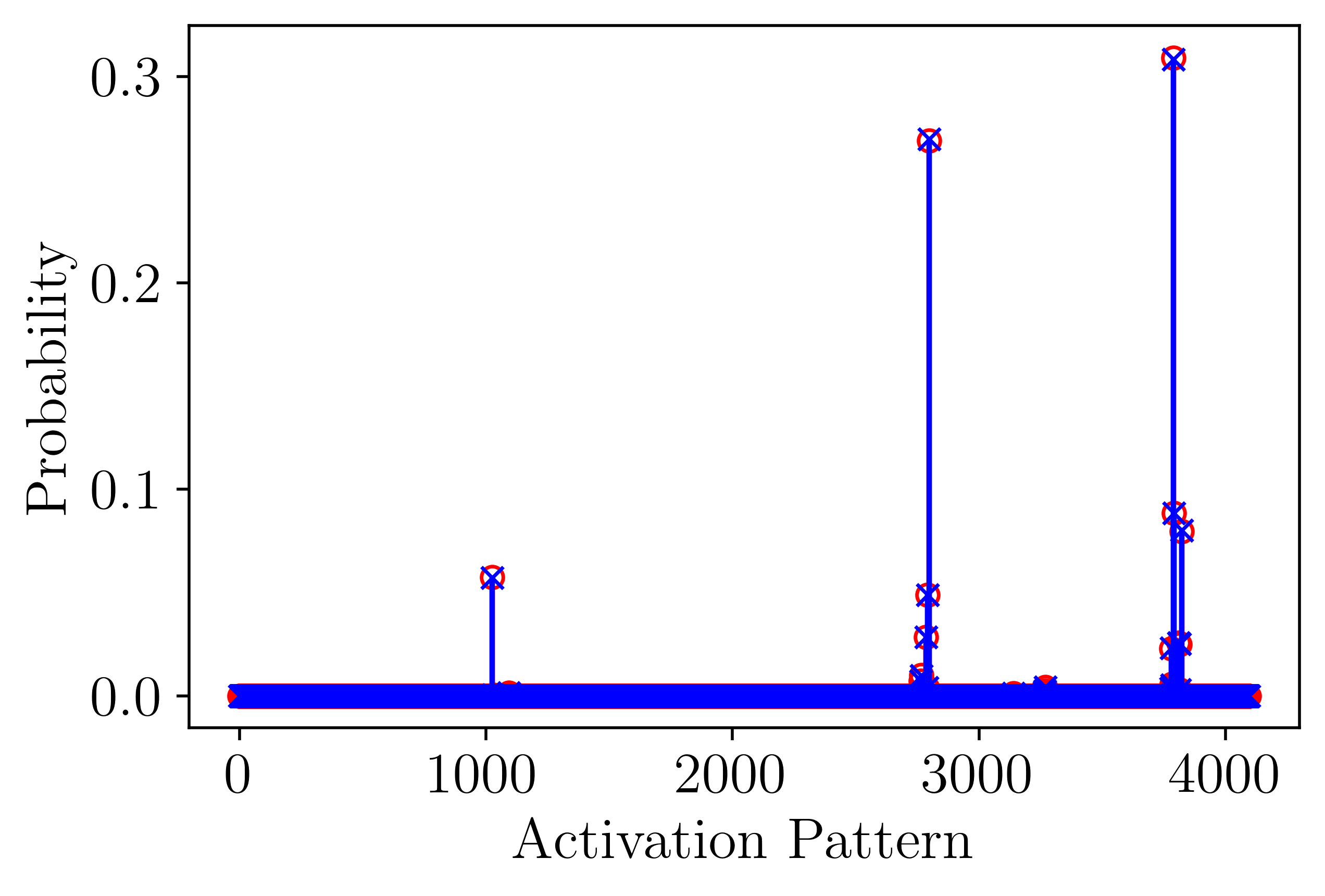} 
    \end{subfigure} 
    \hfill
    \begin{subfigure}[c]{0.32\textwidth}
        \includegraphics[width=\linewidth]{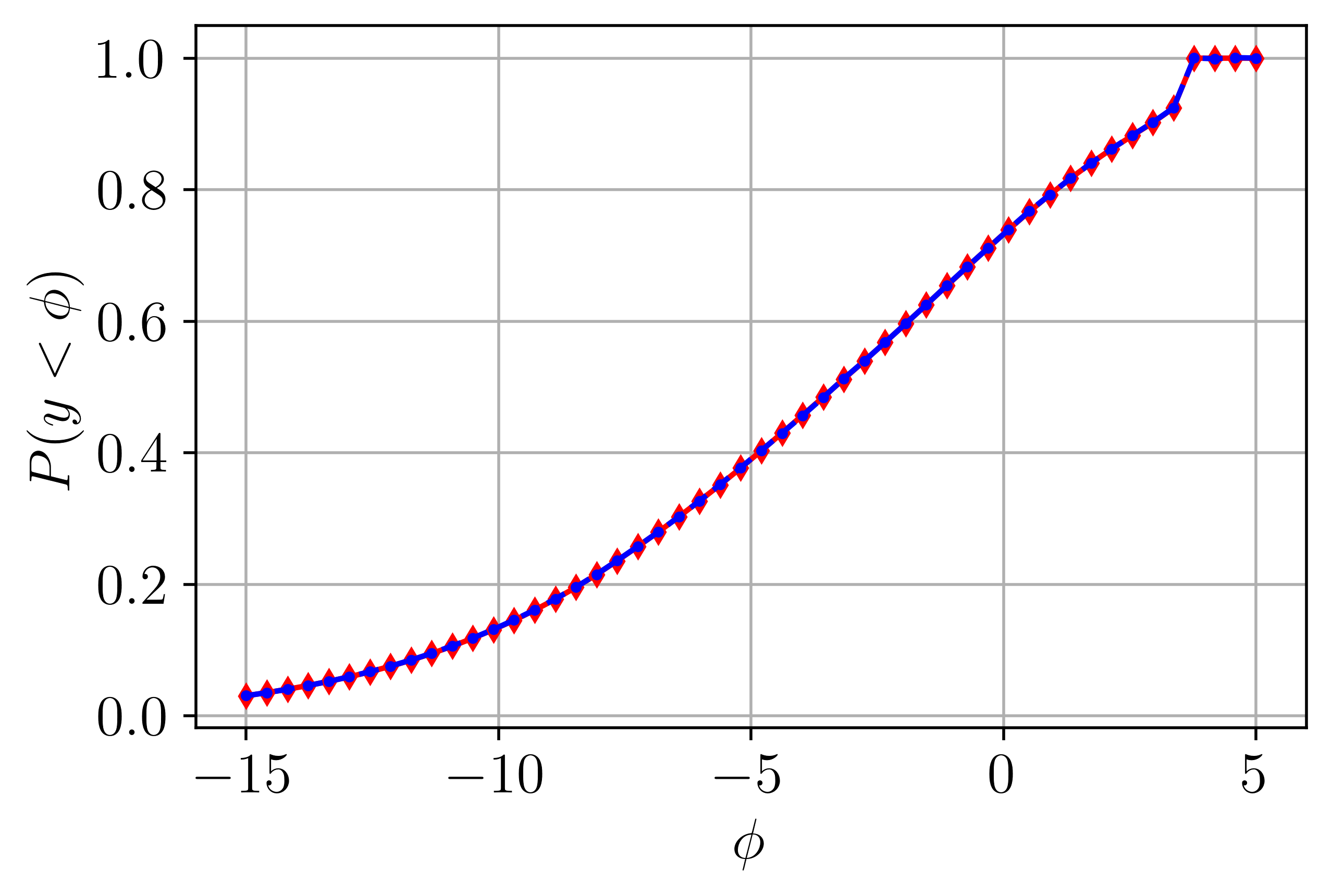} 
    \end{subfigure}
   {\small (b) $\x_2, \sigma^2 = 0.5$ (Blue lines are numerically computed, red lines are Monte Carlo estimates)}
    
    \caption{Effect of Gaussian noise on affine function distribution and output distribution. Blue lines are the numerically computed expressions and red lines are Monte Carlo estimates. [Left Column]: Gaussian distribution input to the network. [Middle Column]: Resulting distribution over activation patterns (affine functions). Each binary activation pattern is denoted by its decimal representation. [Right Column]: CDF of the network output.}
    \label{fig:moons_isotropic}
\end{figure*}

Next we consider the PMF of the activations and CDF of the outputs over the entire moons dataset. We fit a separate Gaussian mixture model to the data points of each class of the moons dataset. Each Gaussian mixture consists of 3 components with diagonal covariances. For each Guassian mixture, we numerically compute the PMF of the activations and the CDF of the outputs. 

The results are shown in \Cref{fig:moons_gmm}. The red lines indicate the probabilities computed using 1 million Monte Carlo trials over the moons dataset, while the blue lines show the numerically computed probabilities using the learned Gaussian mixtures. 

\begin{figure}[htpb]
    \centering
    \begin{subfigure}[t]{0.85\textwidth} 
        \centering
        \begin{subfigure}[t]{0.4\textwidth}
            \includegraphics[width=\textwidth]{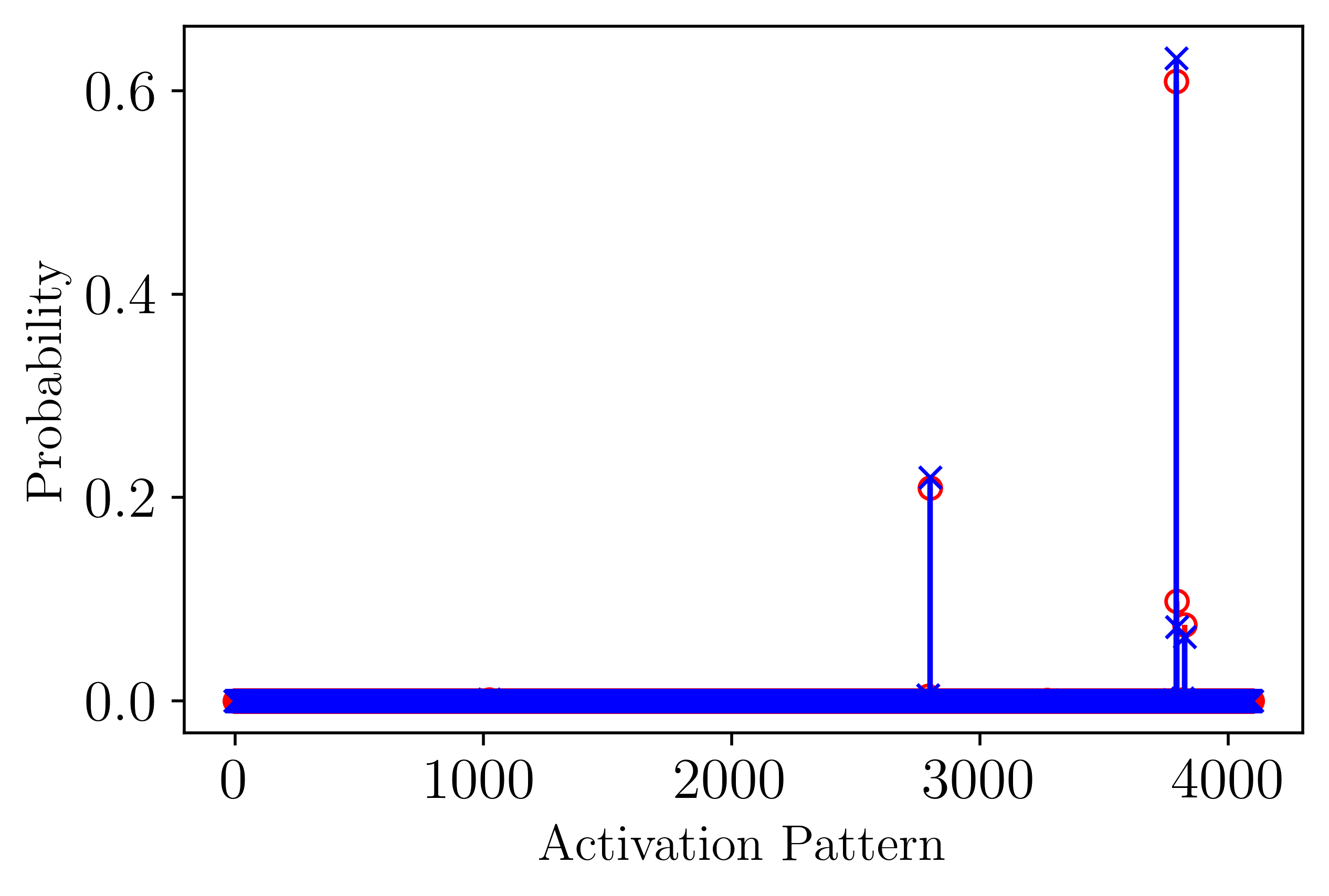}
        \end{subfigure}
        \hspace{2em}
        \begin{subfigure}[t]{0.4\textwidth}
            \includegraphics[width=\textwidth]{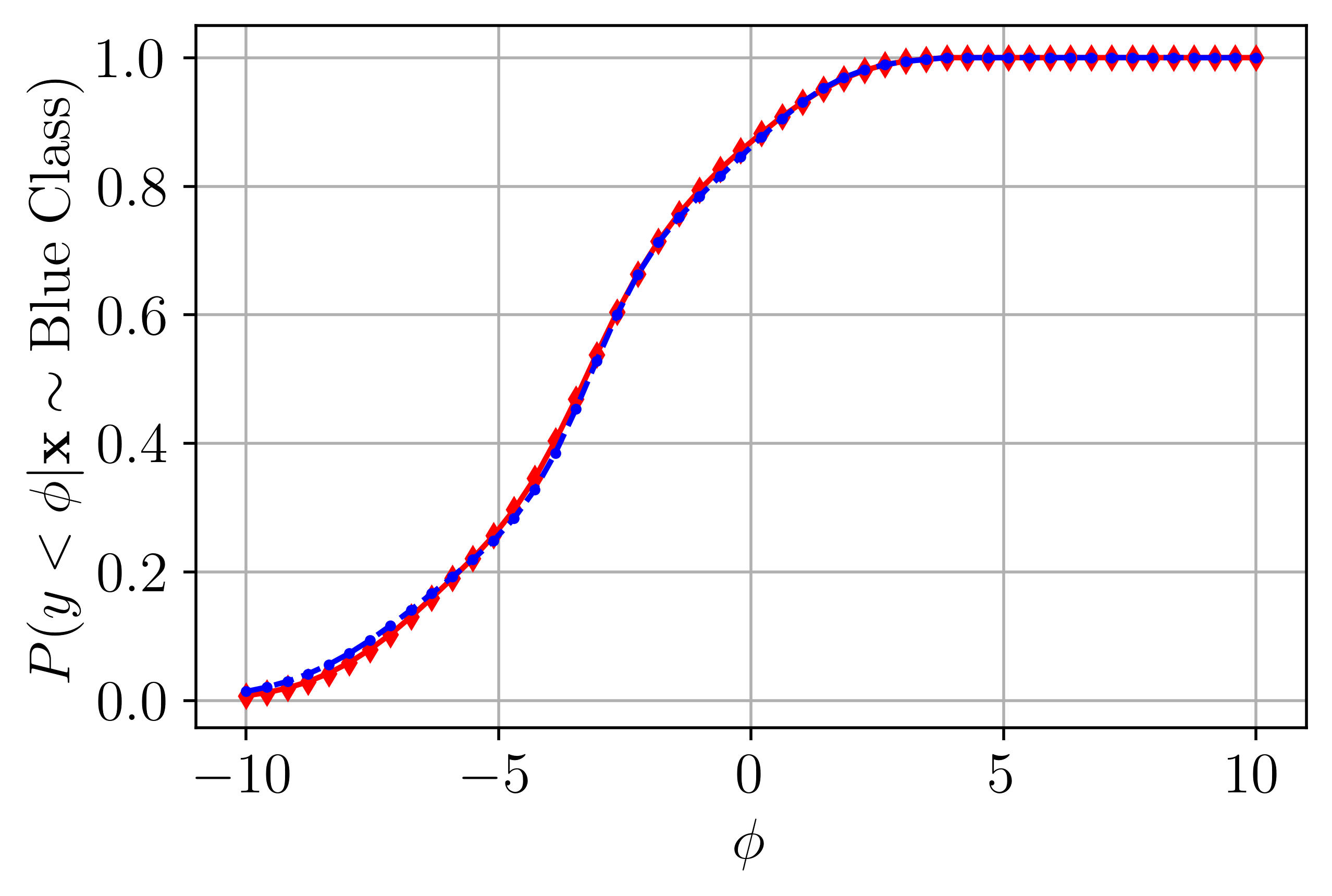}
        \end{subfigure}
        \caption*{(a) Blue Class}
    \end{subfigure}
    \\
    \begin{subfigure}[t]{0.85\textwidth}
    \centering 
    \begin{subfigure}[t]{0.4\textwidth}
        \includegraphics[width=\textwidth]{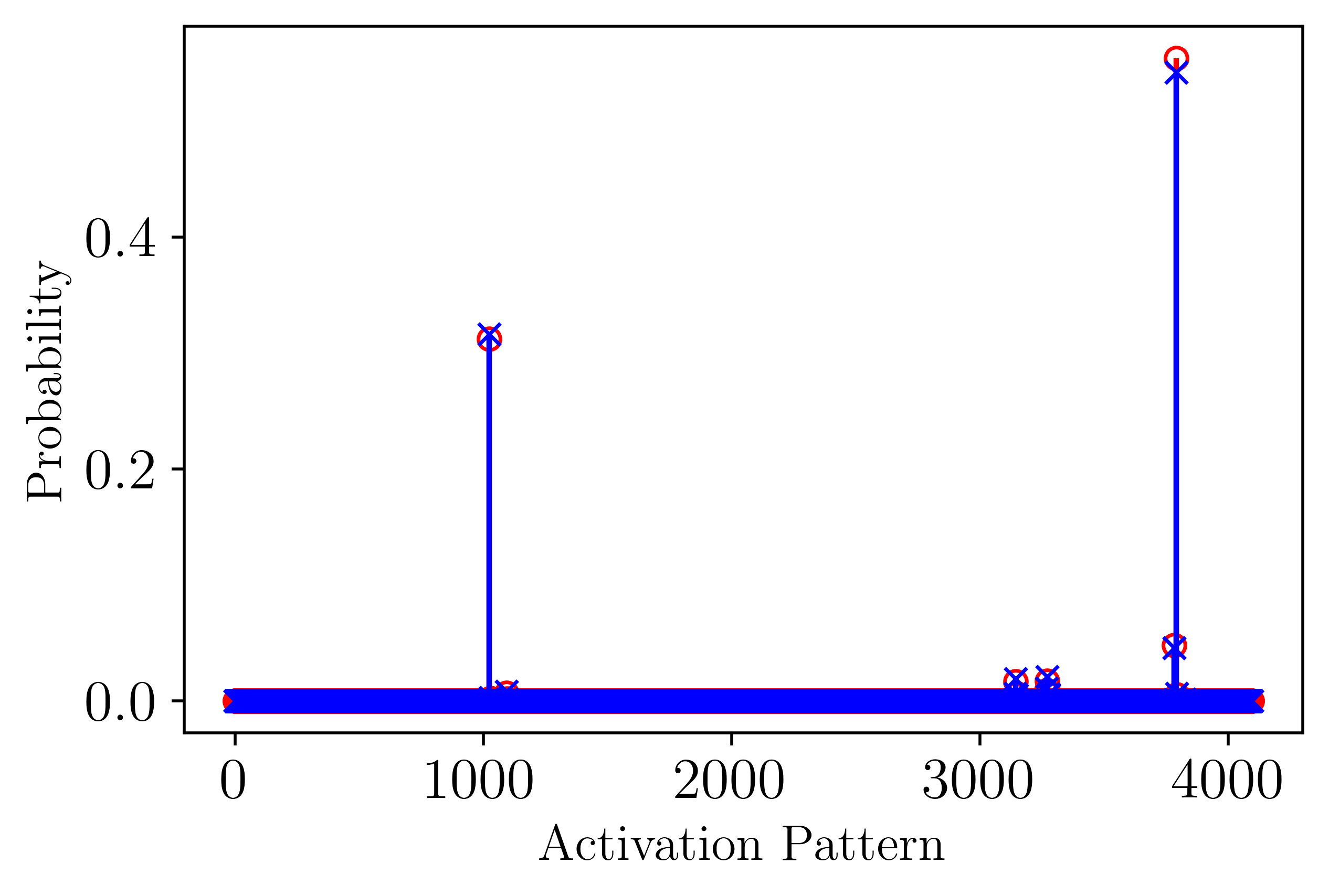}
    \end{subfigure}
    \hspace{2em}
    \begin{subfigure}[t]{0.4\textwidth}
        \includegraphics[width=\textwidth]{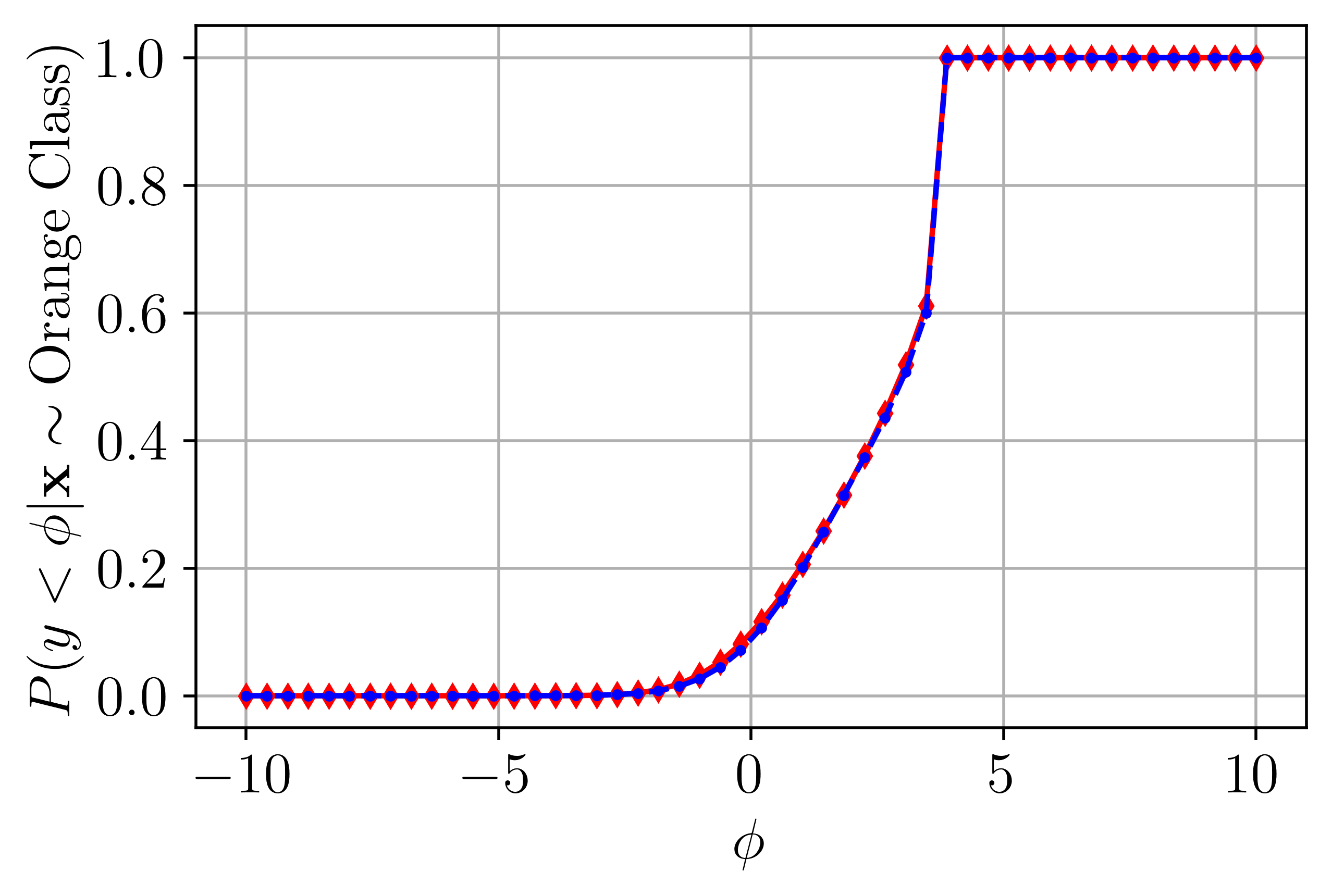} 
    \end{subfigure}    
    \caption*{(b) Orange Class}
    \end{subfigure}
    \caption{PMF of affine functions and CDF of outputs corresponding to Gaussian mixture models fit on: [Left 2] blue class and [Right 2] orange class}
    \label{fig:moons_gmm}
\end{figure}

We observe the numerically computed PMF of the activation patterns usinged the learned Gaussian mixtures closely matches the PMF of activation patterns over the entire moons dataset. The numerically computed false positive and false negative rates 14.49\% and 8.76\%, corresponding to the tail probabilities of the CDF plots in \Cref{fig:moons_gmm}, closely align to the values 13.10\% and 9.78\% computed with Monte Carlo simulation. 

\subsection{Support Estimation}
We validate \Cref{alg:mainalg} to identify the most probable activation patterns for a neural network trained to classify the MNIST dataset \cite{lecun1998mnist}. We train the ReLU network in~\eqref{eq:ff_nn1}-\eqref{eq:ff_nn3} with $L = 4$ and $n_{1} = n_{2} = n_{3} = 16$ neurons per hidden layer. The training details are provided in \Cref{app:experiment_configurations}. For each class, we compute the sample mean and covariance of the vectorized images over the training data. The classes are all modeled as Gaussians with the computed statistics, and we estimate the support of the activation patterns for each class using \Cref{alg:mainalg}. The algorithm returns a set of the most probable activation patterns for each class. We then compute the proportion of activation patterns over the test data that lie in each of the estimated activation sets.

The results are summarized in Table \ref{tab:activation_pattern} below. The thresholds $\tau$ in the table heading are reported in terms of the probability difference from 0.5 rather than entropy. For example, a threshold of 0.1 in the table indicates that neurons with marginal probabilities exceeding 0.6 are fixed to 1, while those below 0.4 in are fixed to 0 in \Cref{alg:mainalg}. A threshold of 0.1 therefore corresponds to deterministically fixing all neurons with entropy below 0.92 bits. In our implementation, we apply a uniform threshold across all layers and limit the cardinality of the index set $\mathcal{I}$ in \Cref{alg:getpatterns} to at most 10 to prevent excessive branching.

Each entry in the table represents the proportion of activation patterns, computed over the test data, that fall within the activation set derived from the sample mean and covariance of the corresponding class and threshold value. Effectively, the table gives estimates of the probability that an affine function applied to a given class comes from the set of affine functions estimated by \Cref{alg:mainalg}. Thus, as the set of activation patterns expands with increasing $\tau$, the probability that the estimated activation set will include an affine function induced by a test sample also increases. 

\begin{table}[ht]
\caption{Success probability of estimated activation pattern sets}
\label{tab:activation_pattern}
\vskip 0.15in
\begin{center}
\begin{small}
\begin{tabular}{ccccc}
\toprule
Class & $\tau=0.1$ & $\tau=0.2$ & $\tau=0.3$ & $\tau=0.4$ \\
\midrule
0 & 0.07 & 0.14 & 0.37 & 1.00\\
1 & 0.04 & 0.18 & 0.70 & 0.85\\
2 & 0.06 & 0.16 & 0.39 & 0.63\\
3 & 0.04 & 0.16 & 0.34 & 1.00\\
4 & 0.01 & 0.16 & 0.38 & 0.74\\
5 & 0.09 & 0.12 & 0.39 & 1.00\\
6 & 0.06 & 0.13 & 0.35 & 0.69\\
7 & 0.03 & 0.31 & 0.41 & 1.00\\
8 & 0.16 & 0.18 & 0.45 & 0.65\\
9 & 0.08 & 0.24	& 0.60 & 0.76\\
\bottomrule
\end{tabular}
\end{small}
\end{center}
\vskip -0.1in
\end{table}

\subsection{Singular value distribution}
We demonstrate that our expressions for the probability of ReLU network realizing an affine function can be applied to analyze the distribution of the singular values of its Jacobian. We train neural networks to classify MNIST and Fashion MNIST, after which we inject noise at given test samples. We then compute the expected activation pattern, and subsequently the expected Jacobian for the networks with inputs given by the noisy distribution.  

For a small ReLU network, computing the singular value distribution \textit{exactly} is feasible by enumerating all affine regions. Here however, we use same ReLU network from the prior experiment, with $L=4$ layers and 32 neurons at each layer. We leverage \Cref{alg:mainalg} with threshold $\tau = 0.25$ to approximately identify the support of the activation pattern corresponding to each noisy distribution. We compare the estimated singular value distribution to that computed using 2,000 Monte Carlo trials. The singular value distributions are computed using 100 randomly sampled test points. For each test point $\x_i$, the input distribution $\N(\x_i, \bSigma)$ is considered, where $\bSigma$ is a randomly initialized covariance matrix with a maximum eigenvalue of 1. The results are illustrated in \Cref{fig:svd_distr} below. 

\begin{figure}[htpb]
    \centering
    \begin{subfigure}[t]{0.4\textwidth}
        \includegraphics[width=\textwidth]{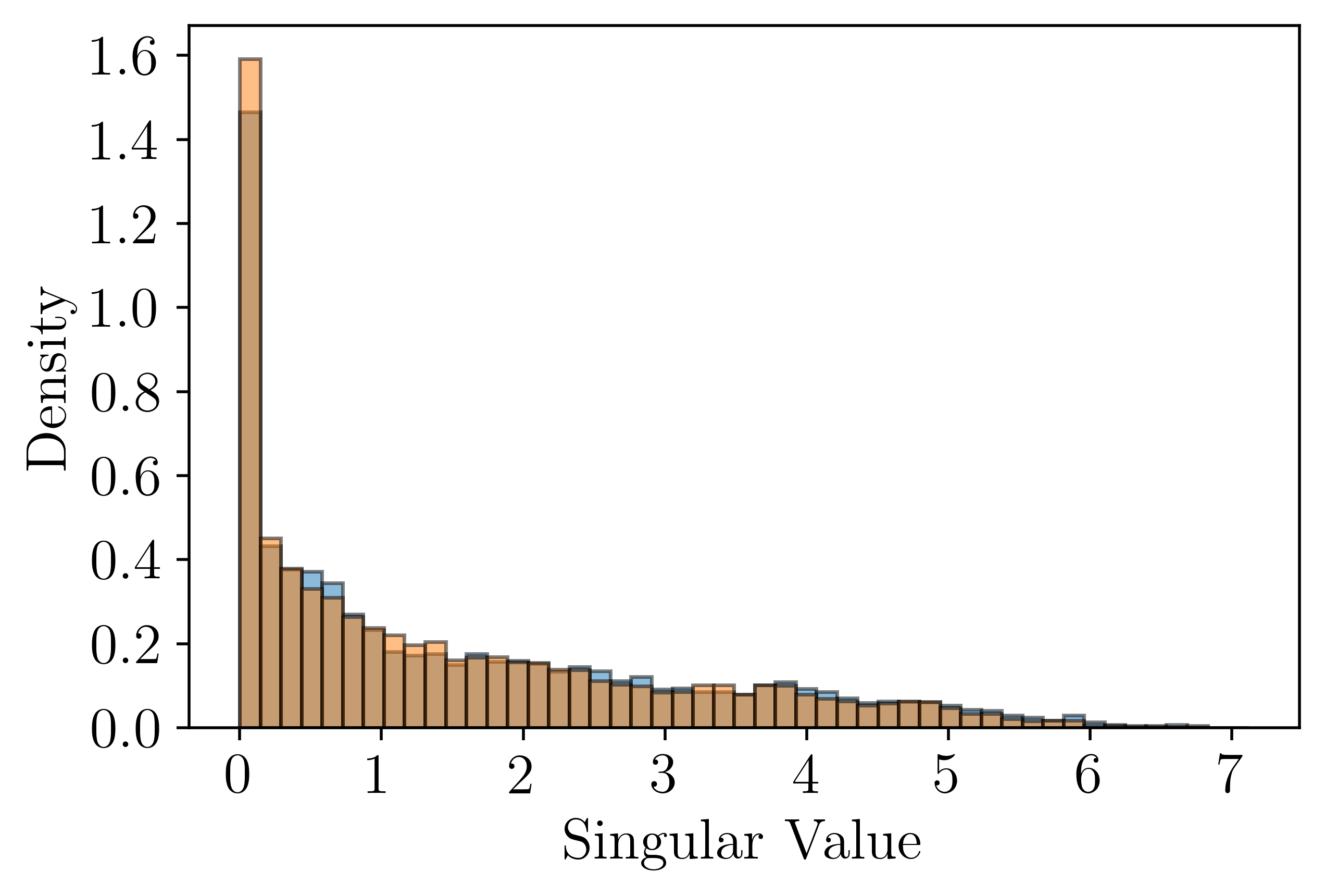} 
        \caption*{(a) MNIST}
    \end{subfigure}
    \hspace{2em} 
    \begin{subfigure}[t]{0.4\textwidth}
        \includegraphics[width=\textwidth]{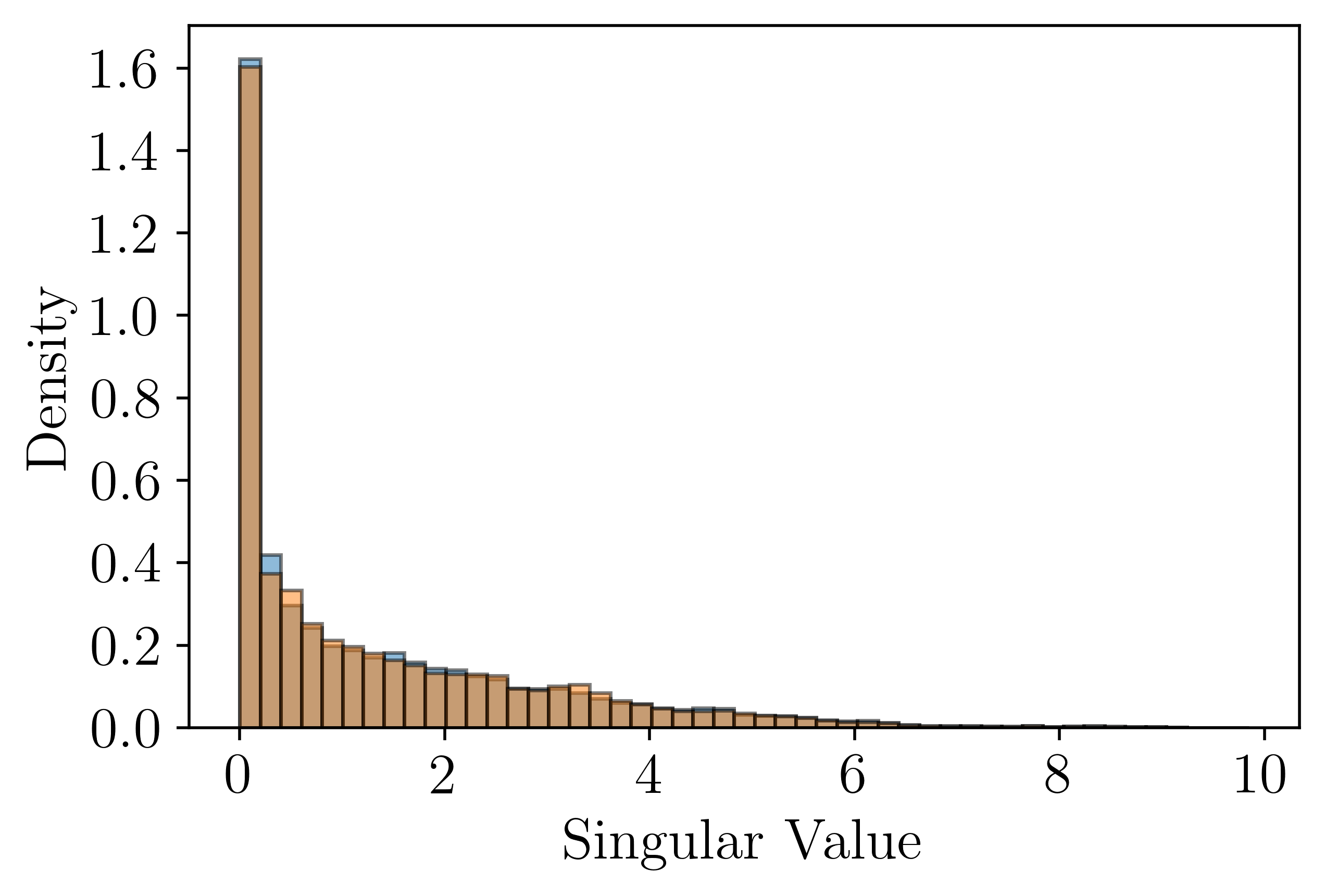} 
        \caption*{(b) Fashion MNIST}
    \end{subfigure}\\ 
    \caption{Singular value distributions. Blue histogram is for numerically computed expression and orange histogram results from Monte Carlo simulation.}
    \label{fig:svd_distr}
\end{figure}

For both datasets we obtain accurate representations of the distribution over the singular values, that even properly characterize the tails of the distributions properly. 

\section{Conclusion}
In this work, we derive the probability that a ReLU function realizes a specific affine transformation from the set of all possible affine functions it can assume. We demonstrate that calculating this probability is equivalent to determining the likelihood of a given activation pattern occurring. We provide explicit, numerically tractable expressions for computing these probabilities and validate their accuracy through Monte Carlo simulations. Additionally, we show that characterizing the distribution over affine functions allows for the derivation of numerically tractable distributions for the output of a neural network. The applications of our work include providing robustness guarantees for piecewise affine neural networks.

\bibliographystyle{IEEEtran}
\bibliography{mybib}

\begin{thebibliography}{10}
\providecommand{\url}[1]{#1}
\csname url@samestyle\endcsname
\providecommand{\newblock}{\relax}
\providecommand{\bibinfo}[2]{#2}
\providecommand{\BIBentrySTDinterwordspacing}{\spaceskip=0pt\relax}
\providecommand{\BIBentryALTinterwordstretchfactor}{4}
\providecommand{\BIBentryALTinterwordspacing}{\spaceskip=\fontdimen2\font plus
\BIBentryALTinterwordstretchfactor\fontdimen3\font minus \fontdimen4\font\relax}
\providecommand{\BIBforeignlanguage}[2]{{%
\expandafter\ifx\csname l@#1\endcsname\relax
\typeout{** WARNING: IEEEtran.bst: No hyphenation pattern has been}%
\typeout{** loaded for the language `#1'. Using the pattern for}%
\typeout{** the default language instead.}%
\else
\language=\csname l@#1\endcsname
\fi
#2}}
\providecommand{\BIBdecl}{\relax}
\BIBdecl

\bibitem{lecun2015deep}
Y.~LeCun, Y.~Bengio, and G.~Hinton, ``Deep learning,'' \emph{Nature}, vol. 521, no. 7553, pp. 436--444, 2015.

\bibitem{cohen2016expressive}
N.~Cohen, O.~Sharir, and A.~Shashua, ``On the expressive power of deep learning: A tensor analysis,'' in \emph{Conference on Learning Theory}.\hskip 1em plus 0.5em minus 0.4em\relax PMLR, 2016, pp. 698--728.

\bibitem{jacot2018neural}
A.~Jacot, F.~Gabriel, and C.~Hongler, ``Neural tangent kernel: Convergence and generalization in neural networks,'' \emph{Advances in Neural Information Processing Systems}, vol.~31, 2018.

\bibitem{kansizoglou2021deep}
I.~Kansizoglou, L.~Bampis, and A.~Gasteratos, ``Deep feature space: A geometrical perspective,'' \emph{IEEE Transactions on Pattern Analysis and Machine Intelligence}, vol.~44, no.~10, pp. 6823--6838, 2021.

\bibitem{bibi2018analytic}
A.~Bibi, M.~Alfadly, and B.~Ghanem, ``Analytic expressions for probabilistic moments of {PL-DNN} with gaussian input,'' in \emph{Proceedings of the IEEE Conference on Computer Vision and Pattern Recognition}, 2018, pp. 9099--9107.

\bibitem{wright2024analytic}
O.~Wright, Y.~Nakahira, and J.~M. Moura, ``An analytic solution to covariance propagation in neural networks,'' in \emph{International Conference on Artificial Intelligence and Statistics}.\hskip 1em plus 0.5em minus 0.4em\relax PMLR, 2024, pp. 4087--4095.

\bibitem{montufar2014number}
G.~F. Montufar, R.~Pascanu, K.~Cho, and Y.~Bengio, ``On the number of linear regions of deep neural networks,'' \emph{Advances in Neural Information Processing Systems}, vol.~27, 2014.

\bibitem{arora2016understanding}
R.~Arora, A.~Basu, P.~Mianjy, and A.~Mukherjee, ``Understanding deep neural networks with rectified linear units,'' \emph{arXiv preprint arXiv:1611.01491}, 2016.

\bibitem{balestriero2018spline}
R.~Balestriero \emph{et~al.}, ``A spline theory of deep learning,'' in \emph{International Conference on Machine Learning}.\hskip 1em plus 0.5em minus 0.4em\relax PMLR, 2018, pp. 374--383.

\bibitem{hinz2021analysis}
P.~Hinz, ``An analysis of the piece-wise affine structure of {ReLU} feed-forward neural networks,'' Ph.D. dissertation, ETH Zurich, 2021.

\bibitem{chu2018exact}
L.~Chu, X.~Hu, J.~Hu, L.~Wang, and J.~Pei, ``Exact and consistent interpretation for piecewise linear neural networks: A closed form solution,'' in \emph{Proceedings of the 24th ACM SIGKDD International Conference on Knowledge Discovery \& Data Mining}, 2018, pp. 1244--1253.

\bibitem{hanin2019deep}
B.~Hanin and D.~Rolnick, ``Deep {ReLU} networks have surprisingly few activation patterns,'' \emph{Advances in Neural Information Processing Systems}, vol.~32, 2019.

\bibitem{lo1972finite}
J.~Lo, ``Finite-dimensional sensor orbits and optimal nonlinear filtering,'' \emph{IEEE Transactions on Information Theory}, vol.~18, no.~5, pp. 583--588, 1972.

\bibitem{Goodfellow-et-al-2016}
I.~Goodfellow, Y.~Bengio, and A.~Courville, \emph{Deep Learning}.\hskip 1em plus 0.5em minus 0.4em\relax MIT Press, 2016, \url{http://www.deeplearningbook.org}.

\bibitem{belisle1993hit}
C.~J. B{\'e}lisle, H.~E. Romeijn, and R.~L. Smith, ``Hit-and-run algorithms for generating multivariate distributions,'' \emph{Mathematics of Operations Research}, vol.~18, no.~2, pp. 255--266, 1993.

\bibitem{chen2018fast}
Y.~Chen, R.~Dwivedi, M.~J. Wainwright, and B.~Yu, ``Fast {MCMC} sampling algorithms on polytopes,'' \emph{Journal of Machine Learning Research}, vol.~19, no.~55, pp. 1--86, 2018.

\bibitem{Stack1962}
G.~P. Steck, ``Orthant probabilities for the equicorrelated multivariate normal distribution,'' \emph{Biometrika}, vol.~49, no. 3/4, pp. 433--445, 1962.

\bibitem{owen2004orthant}
D.~B. Owen, ``Orthant probabilities,'' \emph{Encyclopedia of Statistical Sciences}, vol.~9, 2004.

\bibitem{genz2009computation}
A.~Genz and F.~Bretz, \emph{Computation of multivariate normal and t probabilities}.\hskip 1em plus 0.5em minus 0.4em\relax Springer Science \& Business Media, 2009, vol. 195.

\bibitem{ridgway2016computation}
J.~Ridgway, ``Computation of gaussian orthant probabilities in high dimension,'' \emph{Statistics and Computing}, vol.~26, pp. 899--916, 2016.

\bibitem{azzimonti2018estimating}
D.~Azzimonti and D.~Ginsbourger, ``Estimating orthant probabilities of high-dimensional gaussian vectors with an application to set estimation,'' \emph{Journal of Computational and Graphical Statistics}, vol.~27, no.~2, pp. 255--267, 2018.

\bibitem{rao1973linear}
C.~R. Rao, C.~R. Rao, M.~Statistiker, C.~R. Rao, and C.~R. Rao, \emph{Linear statistical inference and its applications}.\hskip 1em plus 0.5em minus 0.4em\relax Wiley New York, 1973, vol.~2.

\bibitem{abdelaziz2015uncertainty}
A.~H. Abdelaziz, S.~Watanabe, J.~R. Hershey, E.~Vincent, and D.~Kolossa, ``Uncertainty propagation through deep neural networks,'' in \emph{Interspeech 2015}, 2015.

\bibitem{gast2018lightweight}
J.~Gast and S.~Roth, ``Lightweight probabilistic deep networks,'' in \emph{Proceedings of the IEEE Conference on Computer Vision and Pattern Recognition}, 2018, pp. 3369--3378.

\bibitem{krapf2024piecewise}
T.~Krapf, M.~Hagn, P.~Miethaner, A.~Schiller, L.~Luttner, and B.~Heinrich, ``Piecewise linear transformation--propagating aleatoric uncertainty in neural networks,'' in \emph{Proceedings of the AAAI Conference on Artificial Intelligence}, vol.~38, no.~18, 2024, pp. 20\,456--20\,464.

\bibitem{pascanu2013number}
R.~Pascanu, G.~Montufar, and Y.~Bengio, ``On the number of response regions of deep feed forward networks with piece-wise linear activations,'' \emph{arXiv preprint arXiv:1312.6098}, 2013.

\bibitem{telgarsky2015representation}
M.~Telgarsky, ``Representation benefits of deep feedforward networks,'' \emph{arXiv preprint arXiv:1509.08101}, 2015.

\bibitem{raghu2017expressive}
M.~Raghu, B.~Poole, J.~Kleinberg, S.~Ganguli, and J.~Sohl-Dickstein, ``On the expressive power of deep neural networks,'' in \emph{International Conference on Machine Learning}.\hskip 1em plus 0.5em minus 0.4em\relax PMLR, 2017, pp. 2847--2854.

\bibitem{serra2018bounding}
T.~Serra, C.~Tjandraatmadja, and S.~Ramalingam, ``Bounding and counting linear regions of deep neural networks,'' in \emph{International Conference on Machine Learning}.\hskip 1em plus 0.5em minus 0.4em\relax PMLR, 2018, pp. 4558--4566.

\bibitem{hu2020analysis}
Q.~Hu, H.~Zhang, F.~Gao, C.~Xing, and J.~An, ``Analysis on the number of linear regions of piecewise linear neural networks,'' \emph{IEEE Transactions on Neural Networks and Learning Systems}, vol.~33, no.~2, pp. 644--653, 2020.

\bibitem{piwek2023exact}
P.~Piwek, A.~Klukowski, and T.~Hu, ``Exact count of boundary pieces of relu classifiers: Towards the proper complexity measure for classification,'' in \emph{Uncertainty in Artificial Intelligence}.\hskip 1em plus 0.5em minus 0.4em\relax PMLR, 2023, pp. 1673--1683.

\bibitem{goujon2024number}
A.~Goujon, A.~Etemadi, and M.~Unser, ``On the number of regions of piecewise linear neural networks,'' \emph{Journal of Computational and Applied Mathematics}, vol. 441, p. 115667, 2024.

\bibitem{hanin2019complexity}
B.~Hanin and D.~Rolnick, ``Complexity of linear regions in deep networks,'' in \emph{International Conference on Machine Learning}.\hskip 1em plus 0.5em minus 0.4em\relax PMLR, 2019, pp. 2596--2604.

\bibitem{serra2020empirical}
T.~Serra and S.~Ramalingam, ``Empirical bounds on linear regions of deep rectifier networks,'' in \emph{Proceedings of the AAAI Conference on Artificial Intelligence}, vol.~34, no.~04, 2020, pp. 5628--5635.

\bibitem{patel2024local}
N.~Patel and G.~Montufar, ``On the local complexity of linear regions in deep {ReLU} networks,'' \emph{arXiv preprint arXiv:2412.18283}, 2024.

\bibitem{whitaker2023synaptic}
T.~Whitaker and D.~Whitley, ``Synaptic stripping: How pruning can bring dead neurons back to life,'' in \emph{2023 International Joint Conference on Neural Networks (IJCNN)}.\hskip 1em plus 0.5em minus 0.4em\relax IEEE, 2023, pp. 1--8.

\bibitem{lecun1998mnist}
Y.~LeCun, ``The mnist database of handwritten digits,'' \emph{http://yann. lecun. com/exdb/mnist/}, 1998.

\end{thebibliography}

\appendix
\newpage 
\section{Extension to Leaky ReLU}
\label{app:leakyrelu}
The leaky ReLU activation is defined element-wise as
\begin{align}
    \mathrm{LeakyReLU}(\x)_i = \begin{cases}
        x_i &\text{ if } x_i \geq 0\\
        -\gamma x_i &\text{ otherwise} 
    \end{cases}
\end{align}
where $\gamma > 0$ controls the magnitude of the negative slope. When $\gamma = 0$, the leaky ReLU activation is identical to the ReLU activation. All our results hold for the case of leaky ReLU activations with slight modification. Instead of defining $\z_{\ell} = \1(\h_{\ell})$ as we did for ReLU activations, we can redefine $\z_{\ell} = \pi(\h_{\ell})$ where the function $\pi$ is:
\begin{align}
    \pi(\x)_i &= \begin{cases}
        1 &\text{ if } x_i > 0\\
        -\gamma &\text{otherwise} 
    \end{cases}
\end{align}
Replacing all occurrences of $\1(\cdot)$ in the case of ReLU activations with $\pi(\cdot)$ then enables us to easily generalize the results to the leaky ReLU activation function. The results can indeed generally be extended to any activation function $\sigma(\cdot)$ that is piecewise linear and continuous, by expanding the definition of the function $\pi$ so that $\pi(x_i)$ yields the slope that $\sigma(\cdot)$ applies to $x_i$. For a piecewise linear activation $\sigma(\cdot)$ with $k$ linear knots there will exist $k^{\sum_{\ell=1}^{L}}$ distinct activation patterns. In this case however, \Cref{lemma:polytope} holds only if $\sigma(\cdot)$ applies a unique slope to each piecewise region. For example, if $\sigma(\cdot)$ applies the same slope on multiple disjoint regions, then it can be easily shown that the convex polytopes corresponding to a given activation pattern are not unique. Furthermore, in the case of general piecewise linear activations $\sigma(\cdot)$, the region of integration defined by $\mathcal{O}(\bzeta) = \{\x \in \R^n : \pi(\x) = \bzeta\}$ may not be an orthant of $\R^n$ but will still be a rectangular region.
\newpage

\section{Experiment Details}
\label{app:experiment_configurations}

\subsection{Moons}
\label{app:experiment_configurations_moons}
\begin{table}[ht]
\caption{Moons experiment configuration}
\label{sample-table}
\vskip 0.15in
\begin{center}
\begin{small}
\begin{tabular}{lcccr}
\toprule
Parameter &  \\
\midrule
Training Samples    & 1000  \\
Noise Std Dev & 0.2\\
Optimizer & Adam (momentum parameters $\beta_1 = 0.9, \beta_2 = 0.999$)\\
Batch size & 64\\
Learning rate & 1e-2\\
Epochs & 20\\
\bottomrule
\end{tabular}
\end{small}
\end{center}
\vskip -0.1in
\end{table}
Noise Std Dev represents the standard deviation of zero-mean Gaussian noise used to generate the moons dataset.

\subsection{MNIST}
\label{app:experiment_configurations_mnist}
\begin{table}[ht]
\caption{MNIST experiment configuration}
\label{sample-table}
\vskip 0.15in
\begin{center}
\begin{small}
\begin{tabular}{lcccr}
\toprule
Parameter &  \\
\midrule
Optimizer & Adam (momentum parameters $\beta_1 = 0.9, \beta_2 = 0.999$)\\
Batch size & 128\\
Learning rate & 1e-3\\
Epochs & 10\\
\bottomrule
\end{tabular}
\end{small}
\end{center}
\vskip -0.1in
\end{table}

\end{document}